\documentclass[sigconf]{acmart}
\AtBeginDocument{%
  }

\copyrightyear{2025}  
\acmYear{2025}  
\setcopyright{acmlicensed}  
\acmConference[KDD '25]{Proceedings of the 31st ACM SIGKDD Conference on Knowledge Discovery and Data Mining V.2}{August 3--7, 2025}{Toronto, ON, Canada.}  
\acmBooktitle{Proceedings of the 31st ACM SIGKDD Conference on Knowledge Discovery and Data Mining V.2 (KDD '25), August 3--7, 2025, Toronto, ON, Canada}  
\acmISBN{979-8-4007-1454-2/2025/08}  
\acmDOI{10.1145/XXXXX.XXXXXX}
\usepackage{subcaption}
\usepackage{algorithm}
\usepackage{algorithmicx}
\usepackage{algpseudocode}
\usepackage{epstopdf}
\usepackage{multirow}
\usepackage{enumitem}

\usepackage{makecell}
\usepackage{color, colortbl}

\definecolor{Gray}{gray}{0.9}

\begin{document}


\title{Score-based Generative Modeling for Conditional
\\Independence Testing}

\author{Yixin Ren}
\email{yxren21@m.fudan.edu.cn}
\affiliation{%
\institution{Fudan University}
\state{Shanghai}
\country{China}}
\authornotemark[1]

\author{Chenghou Jin}
\email{jinch24@m.fudan.edu.cn}
\affiliation{%
\institution{Fudan University}
\state{Shanghai}
\country{China}}
\authornote{Co-first authors.}

\author{Yewei Xia}
\email{ywxia23@.m.fudan.edu.cn}
\affiliation{%
\institution{Fudan University}
\state{Shanghai}
\country{China}}

\author{Li Ke}
\email{keli.kl@alibaba-inc.com}
\affiliation{%
\institution{Alibaba Group}
\state{Hangzhou}
\country{China}}

\author{Longtao Huang}
\email{kaiyang.hlt@alibaba-inc.com}
\affiliation{%
\institution{Alibaba Group}
\state{Hangzhou}
\country{China}}

\author{Hui Xue}
\email{hui.xueh@alibaba-inc.com}
\affiliation{%
\institution{Alibaba Group}
\state{Hangzhou}
\country{China}}

\author{Hao Zhang}
\email{h.zhang10@siat.ac.cn}
\affiliation{%
    \institution{SIAT, Chinese Academy of Sciences}
\state{Shenzhen}
\country{China}}

\author{Jihong Guan}
\email{jhguan@tongji.edu.cn}
\affiliation{%
\institution{Tongji University}
\state{Shanghai}
\country{China}}
\authornotemark[2]

\author{Shuigeng Zhou}
\email{sgzhou@fudan.edu.cn}
\affiliation{%
\institution{Fudan University}
\state{Shanghai}
\country{China}}
\authornote{Corresponding authors.}

\renewcommand{\shortauthors}{Yixin Ren et al.}
\begin{abstract}
Determining conditional independence (CI) relationships between random variables is a fundamental yet challenging task in machine learning and statistics, especially in high-dimensional settings. Existing generative model-based CI testing methods, such as those utilizing generative adversarial networks (GANs), often struggle with undesirable modeling of conditional distributions and training instability, resulting in subpar performance. To address these issues, we propose a novel CI testing method via score-based generative modeling, which achieves precise Type I error control and strong testing power. Concretely, we first employ a sliced conditional score matching scheme to accurately estimate conditional score and use Langevin dynamics conditional sampling to generate null hypothesis samples, ensuring precise Type I error control. Then, we incorporate a goodness-of-fit stage into the method to verify generated samples and enhance interpretability in practice. We theoretically establish the error bound of conditional distributions modeled by score-based generative models and prove the validity of our CI tests. Extensive experiments on both synthetic and real-world datasets show that our method significantly outperforms existing state-of-the-art methods, providing a promising way to revitalize generative model-based CI testing. 
\end{abstract}





\maketitle
\newtheorem{assumption}{Assumption}
\section{Introduction}
\label{intro}
Given three random variables $X$, $Y$ and $Z$, conditional independence (CI) testing aims to determine whether the null hypothesis $\mathbb{P}_{XY|Z} = \mathbb{P}_{X|Z}\mathbb{P}_{Y|Z}$ can be rejected, assuming the existence of conditional distribution functions. Determining CI relationships has wide applications in various areas including machine learning~\citep{pogodin2022efficient}, causal inference~\citep{Spirtes1993CausationPA} and variable selection~\citep{candes2018panning}. 

Generally, CI testing is a challenging task~\citep{shah2020hardness}. The ``curse of dimensionality" makes the task even more difficult~\citep{scetbon2022asymptotic}. Many methods~\citep{bellot2019conditional,polo2023conditional,scetbon2022asymptotic,li2023nearest} have been proposed to address this challenge. For more detailed information, please refer to the ``Related Work'' section. Among them, conditional randomization-based methods are the de facto choices to address this problem. The conditional randomization tests (CRT) framework~\citep{candes2018panning} reformulates CI testing as a two-sample testing problem~\citep{gretton2012kernel}. 
This transformation provides flexibility, enabling the integration of various test statistics~\citep{gretton2009fast}. Despite its advantages, controlling Type I error in the CRT framework remains a bottleneck~\citep{berrett2020conditional}. This difficulty comes from the need for an accurate estimate of the conditional distribution $\mathbb{P}_{X|Z}$, which becomes increasingly challenging in high-dimensional settings. Generative model-based methods~\citep{bellot2019conditional, shi2021double} attempt to tackle this issue by leveraging neural networks, which excel in scalability and can estimate heterogeneous conditional distributions. However, many generative models, such as Generative Adversarial Networks (GANs), have obvious limitations like mode collapse and instability in training due to adversarial optimization~\citep{arjovsky2017wasserstein}. As a result, these methods fail to achieve reliable CI testing results in practice, as demonstrated in our experiments. 

Recently, score-based generative models, also known as diffusion models, have gained prominence in the field of deep generative modeling~\citep{song2019generative, ho2020denoising}. Unlike GANs, score-based generative models provide stable training and are grounded on solid theoretical foundations, enabling high-quality synthesis. Motivated by their merits, in this paper we explore score-based generative modeling for CI testing, aiming to overcome the limitations of existing generative model-based methods and thus achieve reliable Type I error control and outstanding testing power.

To this end, we propose a novel CI testing method leveraging score-based generative modeling to generate null hypothesis samples. We first introduce a sliced conditional score matching scheme for effective modeling of conditional scores in high-dimensional settings. This approach is then integrated with Langevin dynamics conditional sampling to generate null hypothesis samples, which allow the calculation of the null hypothesis distribution of the test statistic. Then, by comparing with the test statistic computed on the observed data, we can make hypothesis decision on CI testing. Following that, to enhance reliability, we incorporate a goodness-of-fit stage into the method, ensuring the validity of generated samples and improving interpretability in practice. We call the proposed method \textbf{SGMCIT} --- the abbreviation of \textbf{S}core-based \textbf{G}enerative \textbf{M}odeling for \textbf{CI} \textbf{T}esting. 

The novelty of our work is three-fold: 1) Model novelty: we innovatively \textit{extend score matching and Langevin dynamics sampling techniques to the conditional case}; 2) Framework novelty: we firstly \textit{introduce a goodness-of-fit stage into the CRT framework} to ensure CI test reliability in practice; 3) Theoretical novelty: we rigorously \textit{derive an asymptotic Type I error bound for CI tests}. 

\noindent\textbf{Contributions.} In summary, our contributions are as follows:
\begin{itemize}[leftmargin=5mm]
    \item We propose a novel CI testing method via score-based generative modeling, where we design a sliced conditional score matching scheme to effectively model conditional score, which is seamlessly combined with Langevin dynamics-based conditional sampling to generate null hypothesis samples. 
   \item 
    We adopt a goodness-of-fit stage to ensure the validity of the generated samples and further improve the practical interpretability of our CI testing.
    \item We theoretically establish the error bound of conditional distributions modeled by score-based generative models and derive the asymptotic Type I error bound of CI testing, which guarantee the effectiveness of the proposed method.
    \item We conduct extensive experiments on both synthetic and real-world datasets, showing that our method achieves state-of-the-art results across diverse scenarios, and provides an effective way to revitalize generative model-based CI testing.
\end{itemize}  

\noindent \textbf{Outline.} The rest of the paper is organized as follows: Sec.~\ref{sec:related works} reviews related work. Sec.~\ref{sec:preliminaries} briefs the hypothesis testing framework for conditional independence. Sec.~\ref{section_proposedmethods} introduces our CI testing method. 
Sec.~\ref{sec:theoretical_analysis} presents theoretical results. Sec.~\ref{sec:experimental} evaluates the proposed method. Finally, Sec.~\ref{sec:conclusion} concludes the paper.

\section{Related Work}
\label{sec:related works}

\subsection{Score-Based Generative Models}
Score-based generative models (SGMs)~\citep{song2019generative, song2020score}, also called diffusion models~\citep{sohl2015deep, ho2020denoising} are a cutting-edge technique to generative modeling, offering capability of high-quality sample generation across a variety of domains~\citep{ramesh2022hierarchical, kong2020diffwave}. SGMs operate in two main phases: the first involves score estimation, and the second generates samples by sampling from the estimated score. Various score matching techniques have been proposed for score estimation, including slice-based methods~\citep{song2020sliced} and noise-perturbed methods~\citep{hyvarinen2005estimation, vincent2011connection}. While all are effective in practice, noise-perturbed methods require tuning the noise scale, introducing additional parameters. In the sampling phase, Langevin dynamics guided by the estimated score function is commonly used to generate data. Recently, the stochastic differential equations (SDEs) framework~\citep{song2020score} provides new perspectives, which have subsequently inspired later works~\citep{song2021maximum} to better stabilize training, leading to various methods~\citep{song2023consistency, song2023improved} for speeding up the sampling step of generation. 

Although score-based generative modeling has achieved impressive results in many domains, there is no work effectively employing it in CI testing task, which is done in this paper.

\subsection{Conditional Independence Testing}
Conditional independence (CI) testing has been the focus of extensive research, resulting in a wide array of methods tailored to diverse data scenarios~\citep{huber2015test, mittag2018nonparametric, li2020nonparametric, park2020measure, cai2022distribution, chalupka1804fast, renefficiently}. Roughly, existing approaches can be grouped into four main categories: \textbf{Distance-based tests}: These methods~\citep{wang2015conditional, sheng2019distance, warren2021wasserstein} estimate the conditional characteristic function to determine CI relationships. Although effective in some settings, they have limitations in finding high-dimensional or complex dependencies. \textbf{Regression-based methods}: Based on the generalized covariance measure~\citep{zhang2017feature, shah2020hardness}, these approaches can capture weak conditional dependence. When prior knowledge of the data-generative process is available, regression-based methods can be more  robust~\citep{polo2023conditional, zhang2018measuring}. However, their performance is often limited under model misspecification. \textbf{Kernel-based methods}: These methods~\citep{zhang2012kernel, huang2022kernel, strobl2019approximate} leverage characterizations of conditional independence~\citep{daudin1980partial} to construct kernel-based statistics. Permutation tests~\citep{doran2014permutation} are commonly used to approximate null distributions. Recent advances~\citep{scetbon2022asymptotic} include introducing an asymptotic null statistic that is computationally efficient and follows a standard normal distribution under the null hypothesis. \textbf{Conditional randomization-based tests (CRT)}: As a type of prominent approaches~\citep{candes2018panning, shi2021azadkia}, they transform CI testing into a two-sample testing problem~\citep{gretton2012kernel, sen2017model}, offering flexibility with various statistics~\citep{ren2023multi, azadkia2021simple, lopez2013randomized}. Conditional randomization testing relies on the accurate modeling of the conditional distribution $\mathbb{P}_{X|Z}$ under the Model-X framework, which is challenging in high-dimensional settings and may lead to uncontrolled Type I errors~\citep{javanmard2021pearson}. KNN-based methods~\citep{runge2018conditional, li2023k} have been proposed as a solution, leveraging the local structure of data to model $\mathbb{P}_{X|Z}$ more effectively. Additionally, generative model-based approaches~\citep{bellot2019conditional, shi2021double} have emerged as another alternative. However, recent findings~\citep{li2023nearest} show that GAN-based methods often struggle to model complex conditional distributions effectively, resulting in subpar performance compared to other state-of-the-art techniques. 

In this work, we introduce a novel score-based generative model to address the challenges in modeling conditional distributions and mitigating training instability. Our approach ensures precise Type I error control and strong testing power, achieving state-of-the-art results across diverse scenarios.

\section{Preliminaries}
\label{sec:preliminaries}
We begin by introducing the notations and recalling the hypothesis testing framework for CI testing. Let $\mathcal{X}\times \mathcal{Y}\times \mathcal{Z}$ be a separable metric space, typically $\mathbb{R}^{d_x}\times \mathbb{R}^{d_y}\times \mathbb{R}^{d_z}$, we denote  $\mathbb{P}_{XYZ}$ as the Borel probability measure over the random variables $X,Y,Z$ defined on $\mathcal{X}\times \mathcal{Y}\times \mathcal{Z}$, and $\mathbb{P}_{X}, \mathbb{P}_{Y}, \mathbb{P}_{Z}$ are the respective marginal distributions. Let $\mathbb{P}_{X|Z}$ be the conditional distribution, our goal is to determine whether $X\perp \!\!\! \perp Y|Z $ holds if and only if $\mathbb{P}_{XYZ} = \mathbb{P}_{X|Z}\mathbb{P}_{Y|Z}\mathbb{P}_{Z}$. Given $n$ independent and identically distributed ($i.i.d$) samples $\mathcal{D}:=\{(x_i, y_i, z_i)\}^n_{i=1}$ with distribution $\mathbb{P}_{XYZ}$, the hypothesis testing problem is formulated as follows:  
\begin{equation}
\mathcal{H}_0:X\perp \!\!\! \perp Y|Z ~~~~~~~~\text{ 
  versus   }~~~~~~~~ \mathcal{H}_1:X\not\!\perp \!\!\! \perp Y|Z.
\end{equation}
The hypothesis testing for CI is performed in the following steps. First, state the statistic $\rho:\mathcal{X}\times \mathcal{Y}\times \mathcal{Z}\mapsto \mathbb{R}$ and calculate its observed value with $\mathcal{D}$. Then, select a significance level $\alpha$ (typically taking  $0.05$ as value). After that, obtain the $p$-value, which is the probability that the sampling of $\rho$ under null hypothesis $\mathcal{H}_0$ is as extreme as the observed value. Finally, the null hypothesis $\mathcal{H}_0$ is rejected if the $p$-value is not greater than $\alpha$.

Two types of errors may occur during the test. Type I error means the false rejection of $\mathcal{H}_0$, and Type II error indicates when $\mathcal{H}_0$ is wrong but not rejected. A good CI test requires that Type I error rate is upper bounded by $\alpha$ meanwhile Type II error rate is minimized~\citep{zhang2012kernel}. Then, the power of the test (also called test power) is defined by $1$ - Type II error rate.

\begin{figure*}[h]
    \centering
    \includegraphics[scale=0.79]{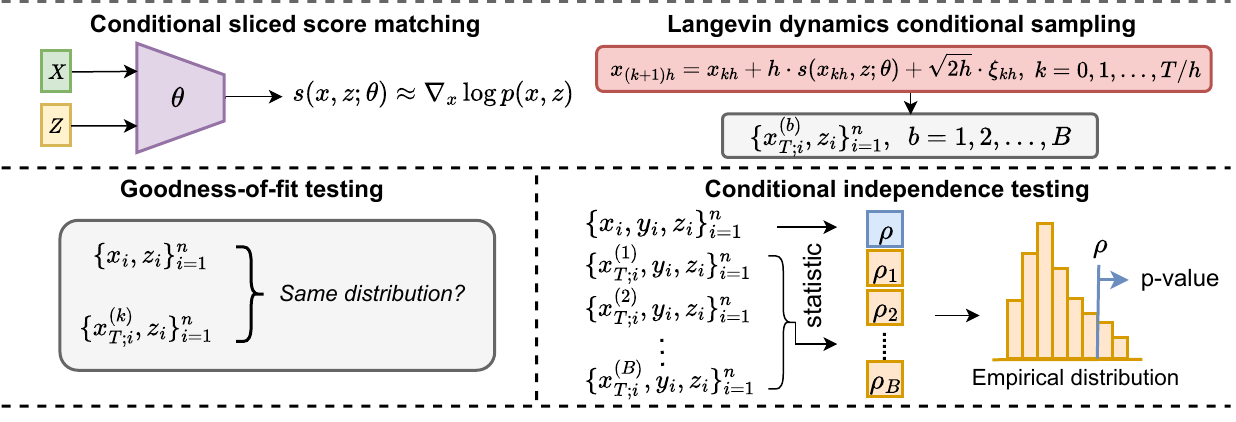}
    \caption{The framework of conditional independence testing with score-based generative modeling.}
    \label{fig:socit}
\end{figure*}

\section{Methodology}
\label{section_proposedmethods}
The framework of our CI testing method SGMCIT (Score-based Generation Modeling for Conditional Independence Testing), is illustrated in Fig.~\ref{fig:socit}. It consists of three major stages. Firstly, a score-based model generates samples to simulate data under the null hypothesis $\mathcal{H}_0$. Secondly, a goodness-of-fit test evaluates the quality of the generated samples, ensuring the reliability of subsequent CI results. Thirdly, CI testing is performed by comparing the test statistic of observed samples with those calculated from the generated samples. 
Due to space limit, some derivation details in the following sections are moved to the Appendix.

\subsection{Score-based Generative Modeling}
\label{subsection_generated_model}
The score-based generative model consists of two stages: score estimation and sampling. For the conditional score estimation, the model is trained with the following objective. 

\noindent \textbf{Conditional sliced score matching objective.} The goal of score estimation is to train a parameterized model $s(x, z;\theta)$ to estimate the real score $\nabla_x \log p(x|z)$ of the data distribution. For any value of $z$, the objective can be written as
\begin{equation*}
\begin{split}
&\mathcal{L}_{\theta}(z) = \frac{1}{2}\mathbb{E}_{x\sim p(x|z)}\left\{\Vert s(x,z;\theta)-\nabla_x \log p(x|z)\Vert_2^2\right\} \\&= \mathbb{E}_{x\sim p(x|z)} \left\{\text{Tr}(\nabla_x s(x,z;\theta))+\frac{1}{2}\Vert s(x,z;\theta)\Vert_2^2\right\}+C(z),    
\end{split}
\end{equation*}
where $\text{Tr}$ is the trace operator and $C(z)$ is a constant that does not depend on $\theta$. However, directly computing $\text{Tr}(\nabla_x s(x,z;\theta))$ is computationally expensive~\citep{song2020sliced}. To address this, techniques such as sliced score matching~\citep{song2020sliced} are employed, which can work well even in the settings where $d_x$ is large. Here, we adopt this technique to the above objective $\mathcal{L}_{\theta}(z)$. We first project the terms $s(x,z;\theta)$ and $\nabla_x \log p(x|z)$ onto some random projection direction $v$ then compare their average differences along that random direction. Formally, let $v\sim p_v$ be the random variable that is independent of $x,z$. Then, we obtain the objective $\mathcal{L}_{\theta}(z;p_v)$ as follows: 
\begin{equation}
\label{slicescore0}
\begin{split}
\mathcal{L}_{\theta}(z;p_v)= \frac{1}{2}\mathbb{E}_{\substack{v\sim p_v\\x\sim p(x|z)}}\left\{\bigl[ v^Ts(x,z;\theta)-v^T\nabla_x \log p(x|z)\bigl]_2^2\right\}. 
\end{split}    
\end{equation}
Under the following three regularity conditions, the dependence of $\mathcal{L}_{\theta}(z;p_v)$ on $\nabla_x \log p(x|z)$ in Eq.~(\ref{slicescore0}) can be eliminated:
\begin{assumption}[Regularity of conditional score functions]
For any $z$, $s(x,z;\theta)$ and $\nabla_x\log p(x|z)$ are both differentiable w.r.t. $x$. 
Additionally, we assume that they satisfy $\mathbb{E}_{x\sim p(x|z)}[\|s(x,z;\theta)\|_2^2] < \infty$ and $\mathbb{E}_{x\sim p(x|z)}[\|\nabla_x\log p(x|z)\|_2^2] < \infty$.
\end{assumption}
\begin{assumption}[Regularity of projection vectors]
The projection vectors satisfy 
$\mathbb{E}_{v\sim p_v}[\|v\|_2^2] < \infty$, and $\mathbb{E}_{v\sim p_v}[vv^T] \succ 0$.
\end{assumption}
\begin{assumption}[Boundary condition]
Given any $z$, for all $\theta \in \Theta$, the score satisfies $\lim_{\|x\| \to \infty} s(x,z;\theta) p(x|z) = 0.$
\end{assumption}
Assumptions 1 and 3 are common assumptions~\citep{song2020sliced}. Also, many distributions for projection vectors satisfy Assumption 2, and in practice, we use $p_v \sim \mathcal{N}(0, I_{d_x})$. Further derivations, using integration by parts as detailed in Appendix~\ref{a1}, yield
\begin{equation}
\begin{split}
\mathbb{E}_{v\sim p_v}\mathbb{E}_{x\sim p(x|z)} \left\{v^T\nabla_x s(x,z;\theta)v+\frac{1}{2}[v^Ts(x,z;\theta)]^2\right\},    
\end{split}    
\end{equation}
which differs from Eq.~(\ref{slicescore0}) only by a constant independent of $\theta$ (see Lemma~\ref{app_lemma_of_loss_nomargin} for more details). To train the model across all values of $z$, the objective is marginalized over $z$, leading to
\begin{equation}
\label{slicescoremargin}
\begin{split}
\mathcal{J}_{\theta} =\mathbb{E}_{\substack{v\sim p_v\\(x,z)\sim p(x,z)}} \left\{v^T\nabla_x s(x,z;\theta)v+\frac{1}{2}[v^Ts(x,z;\theta)]^2\right\}.    
\end{split}    
\end{equation}
In Sec.~\ref{sec:theoretical_analysis}, we will show that $\mathcal{J}_{\theta}$ is valid, indicating that it effectively guides the score model toward the correct solution. In practice, an unbiased estimation objective of $\mathcal{J}_{\theta}$ can be obtained using finite samples from the dataset $\mathcal{D}$, i.e.,
\begin{equation}
\label{lossscore}
\begin{split}
\widehat{\mathcal{J}}_{\theta}=\frac{1}{nm}\sum_{i,j}^{n,m} \Bigl\{v_{ij}^T&\nabla_{x_i} s(x_i,z_i;\theta) v_{ij}^T+\frac{1}{2}[v_{ij}^Ts(x_i,z_i;\theta)]^2\Bigl\},
\end{split}    
\end{equation}
where $\{v_{ij}\}_{1\leq j\leq m}$ represent the $m$ independent projection vectors drawn from $p_v$ for sample $x_i$. The parameter $m$ is selected to trade off variance and computational cost. In our experiments, we find that $m=1$ is already a good choice. By training with this objective function, we obtain the score function of the posterior distribution. Next, we proceed to the sampling stage of generation.

\noindent \textbf{Langevin dynamics conditional sampling.} 
With the estimated score above, we employ Langevin dynamics to generate samples. Let the step size be $h$, the total time be $T$, then for a fixed $z$, the sampling process iteratively updates $x_{kh}$ as follows: 
\begin{equation}
\label{langevin_dynamics_conditional_sampling}
x_{(k+1)h} = x_{kh} + h \cdot s(x_{kh},z; \theta) + \sqrt{2h}\cdot \xi_{kh}, 
\end{equation}
where $\xi_{kh}\sim\mathcal{N}(0,I_{d_x})$ and $x_0\sim \mathcal{N}(0,I_{d_x})$ are for initialization. Intuitively, for a given $z$, the generation process is guided by the conditional score $s(\cdot ,z; \theta)$. Under certain regularity conditions, as $h \to 0$ and $T \to \infty$, $x_T$ converges to a sample from the theoretical distribution~\citep{roberts1996exponential}. A more precise description of this is given in Section~\ref{sec:theoretical_analysis} ``Theoretical Results''. In practice, the error is negligible when $h$ is sufficiently small and $T$ is sufficiently large. Since the procedure does not use any sample of $y$, the generated samples are independent of $y$ given $z$. Consequently, these samples can simulate the samples under null hypothesis $\mathcal{H}_0$. To model the null hypothesis distribution, typically the bootstrap method is employed. Specifically, we generate $B$ sets of pseudo samples, then for the $b$-th set $\{x_{T;i}^{(b)}, z_i\}_{i=1}^n$, the samples are produced iteratively as follows:
\begin{equation}
\label{sampling_B}
x_{(k+1)h;i}^{(b)} = x_{kh;i}^{(b)} +h \cdot s(x_{kh;i}^{(b)},z_i; \theta) + \sqrt{2h}\cdot\xi_{kh}.      
\end{equation}
These pseudo-samples are then used to estimate the null distribution of the statistic. To ensure the reliability of the results, an additional verification step is included to assess the quality of the generated samples, as described in the next subsection.

\noindent \textbf{Further Discussion.} In this paper, we use a neural network as the default score model.  Certainly, other models such as deep kernel exponential families~\citep{wenliang2019learning}, could also be used to introduce smoothing prior. Additionally, we adopt the slicing technique for score matching. While noise-perturbed methods offer an alternative approach, their performance is highly sensitive to the noise scale, and due to the limitation of numerical precision, these methods typically yield results corresponding to data with small noise perturbations. This can, in turn, affect the outcome of CI tests. For sampling, we employ a straightforward implementation. Certainly, more sophisticated, data-adaptive sampling strategies could further improve the performance. Overall, our conditional score-based generative model provides a practical and flexible approach. Future work can focus on refining this model for real-world applications, tailoring score matching methods to specific data distributions, and enhancing the sampling strategy of generation.

\subsection{Goodness-of-Fit Testing}
\label{subsection_goodnessoffit}
Here we introduce the goodness-of-fit (GOF) testing  procedure to evaluate whether the generated samples accurately represent the null hypothesis sample distribution. As our generative model consists of two parts, the test can be done in two ways: (1) verifying the fit of the score model to the observed samples, or (2)  comparing the fit of the generated samples to the observed data.

Existing goodness-of-fit tests~\citep{chwialkowski2016kernel, jitkrittum2017linear} typically assess whether a distribution aligns with a hypothesized model, often using kernel-based measures or explicit score functions. Some tests, such as those in~\cite{jitkrittum2020testing}, are designed specifically for validating conditional distribution properties. However, these methods usually assume a theoretically fixed score model, making them overly sensitive to small estimation errors in cases like ours, where the score model is approximated. Therefore, evaluating the fit of the generated samples to the observed data is more appropriate for our GOF test task. This approach also aligns directly with our goal of using the generated samples for statistical calculation. Two-sample tests~\citep{gretton2009fast}, which use well-established metrics, are particularly effective for this purpose even in high-dimensional cases. For low-dimensional cases, visualization can further improve the interpretability of results.

\begin{figure}[h]
    \centering
    \includegraphics[scale=0.48]{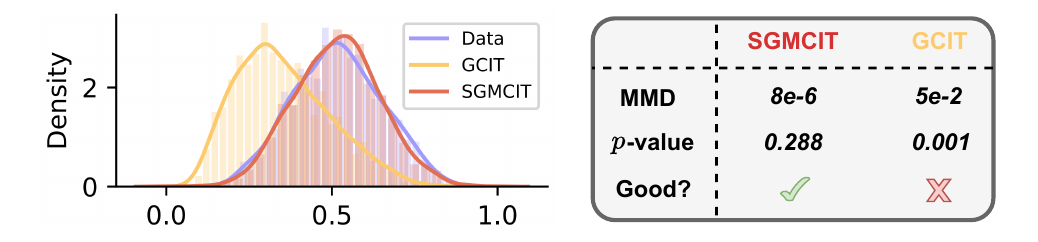}
    \caption{Left: The visualization results for the marginal distribution of $X$ under the chain setting. Right: The corresponding goodness-of-fit testing results using maximum mean discrepancy (MMD).}
    \label{fig:vischain}
\end{figure}

For convenience of understanding, we present an example of experimental results under the chain setting. Fig.~\ref{fig:vischain} shows a comparison of the marginal distribution of $X$ between the generated samples and the observed data, and the corresponding two-sample test results using the kernel-based maximum mean discrepancy (MMD) metric. The GAN-based method, GCIT, generates samples that deviate noticeably from the true data distribution, as indicated by the higher MMD value and larger visual discrepancy. In contrast, our method closely approximates the observed distribution, as evidenced by the smaller MMD value and higher $p$-value, indicating that our method passes the goodness-of-fit test, whereas GCIT fails. These results underscore that our method produces more reliable outcome. In CI test experiments, we further demonstrate that poor modeling of the data distribution will lead to uncontrollable Type I errors, rendering the tests invalid. By accurately modeling the distribution, our approach ensures the validity of CI testing.

\subsection{Conditional Independence Testing}
\label{subsection_conditionalindependencetesting}
Recall that in Sec.~\ref{subsection_generated_model}, we describe the process of generating samples to model the null hypothesis distribution. Specifically, we obtain $B$ sets of pseudo-samples $\{x_{T;i}^{(b)}, z_i\}_{i=1}^n, b\in [B]$~\footnote{The symbol is defined as $[B]:=1,2,...,B$.} through the generative model. These samples maintain the dependence structure of $X$ and $Z$, but interrupt any dependence between $X$ and $Y$. 
The samples are then used to model the null hypothesis distribution by constructing a sequence $\{x_{T;i}^{(b)}, y_i, z_i\}_{i=1}^n, b\in [B]$. Let the corresponding sequence of triples be denoted as $(X^{(b)},Y,Z)_{b=1}^B$. For the observed triples $(X,Y,Z)$, we simplify the notation by referring to it as $(X^{(0)},Y,Z)$, resulting in a combined sequence of triples $(X^{(b)},Y,Z)_{b=0}^B$. The following proposition indicates that this sequence is exchangeable. 
\begin{proposition}[Exchangeablility] Let ${\buildrel d \over =}$ denotes equality in distribution. Then under $\mathcal{H}_0$, and further assume that for all $b\in [B]$, $(X^{(b)},Y,Z)~{\buildrel d \over =}~(X,Y,Z)$, the resulting random sequence of generated triples $(X^{(b)},Y,Z)_{b=0}^B$ is exchangeable. 
\end{proposition}
\begin{proof}
A sequence of random variables is  exchangeable if its distribution is invariant under variable permutations. By the ``representation theorem''~\citep{diaconis1980finite} for exchangeable sequences of random variables, every sequence of conditionally $i.i.d.$ random variables can be considered as a sequence of exchangeable random variables. Recall the process of our generative model, we start from $i.i.d.$ sequence of init random variables $x_0^{(b)}$, then iteratively updates as 
\begin{equation}
x_{(k+1)h}^{(b)} = x_{kh}^{(b)} +h \cdot s(x_{kh}^{(b)},Z; \theta) + \sqrt{2h}\cdot\xi_{kh},      
\end{equation}
where $\xi_{kh}\sim\mathcal{N}(0,I_{d_x})$. Note that for each step $t=kh$, we can represent the generated process of $x_T^{(b)}$ as 
\begin{equation}
\label{generative_process}
x_T^{(b)} = \phi_{T}(\cdots \phi_t(\cdots \phi_1(x_0^{(b)}; Z, \xi_{h}); Z, \xi_{kh}); Z, \xi_{T}),    
\end{equation}
where $\phi_t(x_{(k-1)h}^{(b)}; Z, \xi_{kh}) = x_{(k-1)h}^{(b)} + h \cdot s(x_{(k-1)h}^{(b)},Z; \theta) + \sqrt{2h}\cdot\xi_{kh}$. By the construction of Eq.~(\ref{generative_process}), since the score function $s(\cdot,z;\theta)$ is measurable and the additional noise $\xi_{kh}$ and $Z$ are independent of $x_0^{(b)}$, the resulting random sequence of random variables $(X^{(b)},Y,Z)_{b=1}^B$ is exchangeable according to the ``representation theorem'', thus completes the proof.
\end{proof}
Let the statistic be $\rho:\mathcal{X}\times \mathcal{Y}\times \mathcal{Z}\mapsto \mathbb{R}$, then the $p$-value for the test can be approximated by comparing the statistics of the generated samples with those of the observed samples as follows: 
\begin{equation}
p\text{-value}= \frac{1+\sum_{b=1}^B\mathbf{1}\{\rho(X^{(b)},Y,Z)\geq \rho(X,Y,Z)\}}{1+B},    
\end{equation}
where $\mathbf{1}(\cdot)$ is the indicator function. The exchangeability of sample sequence $(X^{(b)},Y,Z)_{b=0}^B$ implies that the $p$-value is valid, i.e.,
\begin{equation}
\mathbb{P}(p\leq \alpha|\mathcal{H}_0)\leq \alpha, \text{for any given }\alpha\in (0,1).    
\end{equation}
Specifically, by the ``representation theorem''~\citep{diaconis1980finite}, we can show that, for any measurable function $\rho$, a statistic chosen independently of the value of $X$ preserves exchangeability. Consequently, under the null hypothesis $\mathcal{H}_0$, the sequence $[{\rho(X^{(b)}, Y, Z)}]_{b=0}^B$ is also exchangeable, thus we can prove that $p$-value is valid by definition, see Proposition~\ref{app_valid_pvalue} in Appendix~\ref{a3} for details. Note that the previous analyses assume that the generated samples perfectly approximate the true conditional distribution, that is, under $\mathcal{H}_0$, for all $b\in [B]$,  $(X^{(b)},Y,Z)~{\buildrel d \over =}~(X,Y,Z)$ is satisfied. In Sec.~\ref{sec:theoretical_analysis}, we will extend this result to the case without this assumption, demonstrating that our method can control Type I error rate asymptotically.  

\begin{algorithm}
\small
\caption{The SGMCIT method}
\label{alg1}
\begin{flushleft}
\textbf{Input:} data $\mathcal{D}:=\{(x_i, y_i, z_i)\}^n_{i=1}$, significance level $\alpha$, statistic $\rho$, projection number $m$, step size $h$, total time $T$, iterations $B$.  \\
\textbf{Output:}  $ X\perp\!\!\!\perp Y|Z$ or $ X\not\! \perp\!\!\!\perp Y|Z$.
\end{flushleft}
\begin{algorithmic}[1]
\State\textbf{\textit{Stage 1: Score-based generative modeling.}}
\State Train $s(x,z,\theta)$ with $\widehat{\mathcal{J}}_{\theta}$ as Eq.~(\ref{lossscore}). $\lhd$ \textbf{Score estimation}
\State Generate $B$ samples with $s(x,z,\theta)$ as Eq.~(\ref{sampling_B}). $\lhd$ \textbf{Sampling}
\State Obtain $\{x_{T;i}^{(b)}, z_i\}_{i=1}^n$, $b\in [B]$. $\lhd$ \textbf{Generated samples}
\State \textbf{\textit{Stage 2: Goodness-of-fit testing.}}
\State Check whether $\{x_{T;i}^{(b)}, z_i\}_{i=1}^n$ is good enough.
\State \textbf{\textit{Stage 3: Conditional independence testing.}}
\State $\mathcal{D}_{T}^{(b)}:=\{x_{T;i}^{(b)},y_i, z_i\}_{i=1}^n$, $b\in [B]$. $\lhd$ \textbf{Forming sample triples} 
\State $\hat{p}\leftarrow\frac{1+\sum_{b=1}^B\mathbf{1}\{\rho(\mathcal{D}_{T}^{(b)})\geq \rho(\mathcal{D})\}}{1+B}$. $\lhd$ \textbf{Calculating $p$-value}
\State Return $X\not\! \perp\!\!\!\perp Y|Z$ if $\hat{p}\leq \alpha$ holds, otherwise $X\perp\!\!\!\perp Y|Z$.
\end{algorithmic}
\end{algorithm}
\noindent \textbf{Algorithm.} Our algorithm is outlined in Algorithm~\ref{alg1}. SGMCIT consists of three main stages: 1) Obtaining the score $s(x,z,\theta)$ by the conditional sliced score matching objective and generating samples by Langevin dynamics conditional sampling (Lines 1-4). 2) Executing the goodness-of-fit procedure (Lines 5-6) to check whether samples are good enough. 3) By comparing the statistics of the generated samples with those of the observed samples, the $p$-value is calculated (Lines 7-9). Finally, $p$-value and significance level $\alpha$ are used to determine the conditional independence (Line 10). The time complexity for generating triples is $\mathcal{O}(nBT)$, and the complexity for calculating the statistic in Stage 2 and 3 depends on specific choices of statistic, as we will discuss later.  

\noindent \textbf{Choice of statistic.}  Our method boasts  flexibility in combining different statistics tailored to specific scenarios. Notably, the validity of our test, as shown in prior analysis, does not depend on the choice of statistic, provided that the conditional distribution is accurately estimated. However, in the finite sample case, error in conditional distribution estimation is inevitable. To mitigate the impact of error, it is beneficial to select a statistic $\rho$ that is less sensitive to minor discrepancy between the generated and true samples under the null hypothesis $\mathcal{H}_0$. This helps ensure that the Type I error remains within an acceptable bound. In practice, test reliability can be enhanced by using metrics that are somewhat less sensitive in the conditional independence (CI) testing phase than in the goodness-of-fit (GOF) phase. Kernel-based metrics such as Maximum Mean Discrepancy (MMD)~\citep{gretton2009fast}, Hilbert-Schmidt Independence Criterion (HSIC)~\citep{gretton2007kernel, ren2024learning}, and Randomized Dependence Coefficient (RDC)~\citep{lopez2013randomized}, are popular choices. For example, we can use the more robust MMD with $\mathcal{O}(n^2)$ complexity during the GOF phase, while opting for RDC in the CI test phase. RDC offers a computational complexity of $\mathcal{O}(n \log n)$ by leveraging kernel acceleration techniques. This strategy also aligns with GCIT's code implementation, ensuring fairness in experimental comparisons and demonstrating the versatility of our method. 

\section{Theoretical Results}
\label{sec:theoretical_analysis}
In this section, we present major theoretical results of our CI testing method. The results consists of two major parts: (1) the validity of score-based generative modeling, which plays a crucial role in determining the CI testing outcome, and (2) the validity of the proposed CI testing method. Due to space limit, we only give sketch proofs, the full proofs are in the Appendix. 

We begin by analyzing the accuracy of score estimation. Assume that the score function $s(x,z;\theta)$ corresponds to the distribution model $p(x|z;\theta)$. 
Let the parameter $\theta^*$ be the optimal parameter and the parameter space be $\Theta$. 

\begin{assumption}[Identifiability]
\label{ass_identifiabliliy}
The model family $\{p(x,z; \theta) \mid \theta \in \Theta\}$ is well-specified, i.e., $p(x,z) = p(x,z; \theta^*)$. 
Furthermore, $p(x,z; \theta) \neq p(x,z; \theta^*)$ whenever $\theta \neq \theta^*$.
\end{assumption}
\begin{assumption}[Positiveness]
\label{ass_positiveness}
The probability density function satisfies $p(x,z; \theta) > 0, \ \forall \theta \in \Theta \ \text{and} \ \forall (x,z).$ 
\end{assumption}

\begin{lemma} 
\label{lemma_identifiable}
Assume the model family is well-specified and identifiable (Assumption 4), and the densities are positive (Assumption 5). Further under Assumptions 1-3, we have
\begin{equation}
\begin{split}
\mathcal{L}_{\theta}(p_v) := \mathbb{E}_{z\sim p(z)}[\mathcal{L}_{\theta}(z;p_v)] = 0 \Leftrightarrow \theta = \theta^*.
\end{split}
\end{equation}
\end{lemma}
\begin{proof} [Sketch of proof]
We prove the implication from left to right, as the reverse is straightforward. First, using Assumptions~\ref{ass_identifiabliliy} and~\ref{ass_positiveness}, we have $s(x, z; \theta) = \nabla_x \log p(x|z)$. Next, Integrating both sides w.r.t. $x$, we can derive $p(x,z;\theta) =  p(x,z) = p(x,z;\theta^*)$. By the identifiability Assumption~\ref{ass_identifiabliliy}, we conclude that $\theta = \theta^*$. 
\end{proof}

As a result, the process of finding optimal parameters is equivalent to optimizing the loss objective $\mathcal{L}_{\theta}(p_v)$. 
By further ignoring the constant terms that are independent of $\theta$, i.e., the final optimization objective is given by ${\mathcal{J}}_{\theta}$ as Eq.~(\ref{slicescoremargin}). 
Then, the optimal parameter satisfies $\theta^* = \mathop{\arg \min}_{\theta\in\Theta} {\mathcal{J}}_{\theta}$, demonstrating the effectiveness of our designed loss function ${\mathcal{J}}_{\theta}$ in guiding the model toward the optimal solution. More details are provided in the Appendix. In practice, we use a finite approximation of ${\mathcal{J}}_{\theta}$, denote as $\widehat{\mathcal{J}}_{\theta}$ as Eq.~(\ref{lossscore}), and the empirical estimator of the parameter is given by $\hat{\theta}_{n, m} = \mathop{\arg \min}_{\theta\in\Theta} \widehat{\mathcal{J}}_{\theta}$. The following theorem establishes the consistency of this parameter estimator by extending the results of~\citep{song2020sliced} to the case of conditional distributions.
\begin{assumption}[Compactness] The parameter space is compact.
\end{assumption}

\begin{assumption}[Lipschitz continuity]
Both the term $\nabla_x s(x, z; \theta)$ and $s(x, z; \theta) s(x, z; \theta)^T$ are Lipschitz continuous in terms of Frobenious norm, i.e., for all $ \theta_1, \theta_2 \in \Theta$, $||\nabla_x s(x, z; \theta_1) - \nabla_x s(x, z; \theta_2)||_F \leq $\\ $ L_1(x,z)||\theta_1 - \theta_2||_2$, $||s(x, z; \theta_1) s(x, z; \theta_1)^T - s(x, z; \theta_2) s(x, z; \theta_2)^T||_F $\\$\leq L_2(x,z)||\theta_1 - \theta_2||_2$. In addition, we require that the Lipschitz constant satisfies $\mathbb{E}_{(x,z)}[L_1^2(x,z)] < \infty$ and $\mathbb{E}_{(x,z)}[L_2^2(x,z)] < \infty$.
\end{assumption}

\begin{assumption} [Bounded moments of projection vectors]
The moments of projection vectors satisfy 
$\mathbb{E}_{v\sim p_v}[||v v^T||_F^2] < \infty$.
\end{assumption}

\begin{theorem}[Consistency]
\label{consistency_theorem}
Under Assumptions 1-8, 
$\hat{\theta}_{n,m}$ is consistent, meaning that $\hat{\theta}_{n, m} \overset{p}{\to} \theta^*$ as $n\to \infty$, where the notion $\overset{p}{\to}$ denotes convergence by probability.
\end{theorem}

\begin{proof} [Sketch of proof]
According to Lemma~\ref{lemma_identifiable}, we have $\mathcal{J}_{\theta}= 0 \Leftrightarrow \theta = \theta^*$. Then, we can prove the uniform convergence of $\widehat{\mathcal{J}}_{\theta}$, which holds regardless of $m$. These two results lead to consistency. 
\end{proof}

This result above implies that the estimated parameter converges to the optimal parameter as the training sample size increases. By the continuous mapping theorem, this convergence can be extended to the score, i.e., $s(x,z;\hat{\theta}_{n, m}) \overset{p}{\to} s(x,z;\theta^*)$. Controlling the error in the score function ensures that the error in the final distribution remains manageable by appropriately tuning the parameters in the Langevin dynamics conditional sampling (LDCS) process. 

\begin{assumption} [Smoothness] 
\label{ass_smoothness}
For any $z$, $\log p(x|z)$ is continuously differentiable ($C^1$) w.r.t. $x$ and is $L_z$-smooth w.r.t. $x$, meaning that the conditional score function $\nabla_x \log p(x|z)$ is $L_z$-Lipschitz. Additionally, we assume $L_z \geq 1$ for all $z$.
\end{assumption}

\begin{assumption} [Log-Sobolev inequality constraints] 
\label{ass_logSobolev}
For any $z$, we assume that $p(x|z)$ satisfies a log-Sobolev inequality with constant $C_{z;\text{LS}}$. Furthermore, we assume $C_{z;\text{LS}} \geq 1$ for all $z$.
\end{assumption}

Based on the results of~\citep{lee2022convergence}, we have the following error bound:
\begin{theorem}[Error Bound of Conditional Distribution] 
\label{Error bound of conditional distribution}
Under Assumptions 1-10, running LDCS with the estimated score $s(x,z;\hat{\theta}_{n, m})$,  using an appropriate step size $h$, and time $T$, then for any $z$,  yields a conditional distribution $p_{T;n}(x|z)$ such that the total variation (TV) distance of the error satisfies
\begin{equation}
d_\text{TV}\{p_{T;n}(x|z), p(x|z)\} = o_p(1), ~\text{ as } n\to \infty.   
\end{equation}
\end{theorem} 
\begin{proof} [Sketch of proof]
We begin by controlling the estimation error of the score, leveraging the result $s(x,z;\hat{\theta}_{n, m}) \overset{p}{\to} s(x,z;\theta^*)$. Using Assumptions~\ref{ass_smoothness} and~\ref{ass_logSobolev}, we extend the convergence results from~\citep{lee2022convergence}. Specifically, we can show that when the score estimation error is bounded, the error in the conditional distribution, with the appropriate sampling parameters 
$h$ and $T$, is also well-controlled.
\end{proof}

This theorem guarantees that the total variation (TV) distance between the generated and true conditional distributions is asymptotically negligible, ensuring the validity of the CI test. That is, for finite test samples, we show that Type I error remains controllable given sufficient (large $n$) training samples, as stated in the following:
\begin{theorem}[Type I error Bound]
\label{theorem_type_I}
Assume the null hypothesis $\mathcal{H}_0: X\perp \!\!\! \perp Y|Z$ is true. Under Assumptions 1-10, for any significance level $\alpha\in (0,1)$, the bound for the Type I error is given by  
\begin{equation}
\mathbb{P}(p\text{-value}\leq \alpha |\mathcal{H}_0) \leq \alpha + o_p(1), \text{ as } n\to \infty.    
\end{equation}
\end{theorem}
\begin{proof} [Sketch of proof]
We first derive an upper bound on the Type I error rate based on the distribution estimation error. Next, using the asymptotic error control established in Theorem~\ref{Error bound of conditional distribution}, we can show that the upper bound of this Type I error rate diminishes to $\alpha$, leading to the expected result.
\end{proof}

\noindent \textbf{Remark.} Our theoretical analyses above focus on asymptotic properties, which guide practical implementation. The validity of the CI test relies on accurate modeling of conditional distributions, where the goodness-of-fit stage plays a critical role. This stage enhances the confidence in the reliability of the CI test, as further demonstrated in the ``Performance Evaluation'' section.


\section{Performance Evaluation}
\label{sec:experimental}
In this section, we present experimental evaluation that  comprises two major parts: the assessment of our generative model and the evaluation of our CI testing method. Due to space limit, we move the details of implementation, more visualization results, performance across additional metrics, and runtime results to the Appendix. The code is available at:  \url{https://github.com/jinchenghou123/SGMCIT}.

\begin{figure*}[h]
\centering
\includegraphics[scale=0.53]{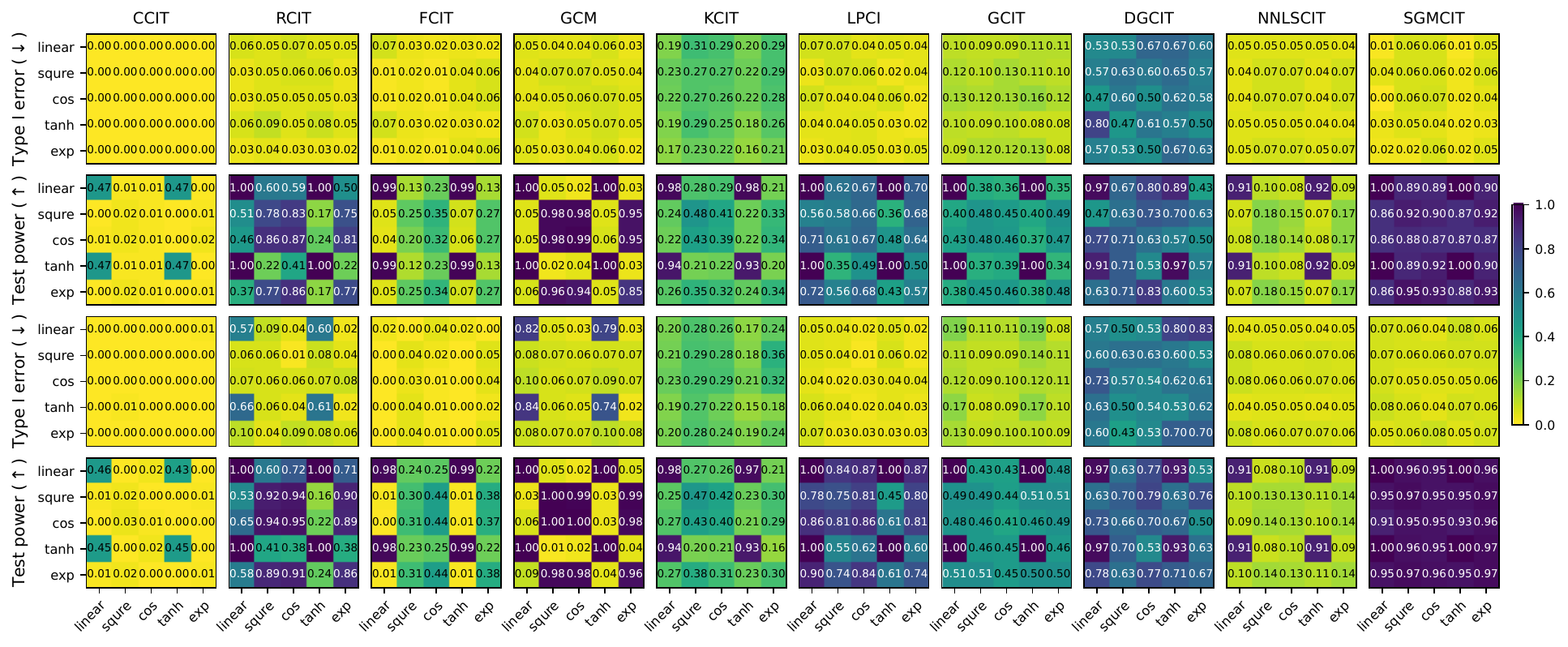}
\caption{
Results of conditional independence tests on benchmark datasets. }
\label{fig:experiments_gaussian}
\end{figure*}

\begin{figure}[h]
    \centering
    \includegraphics[scale=0.4]{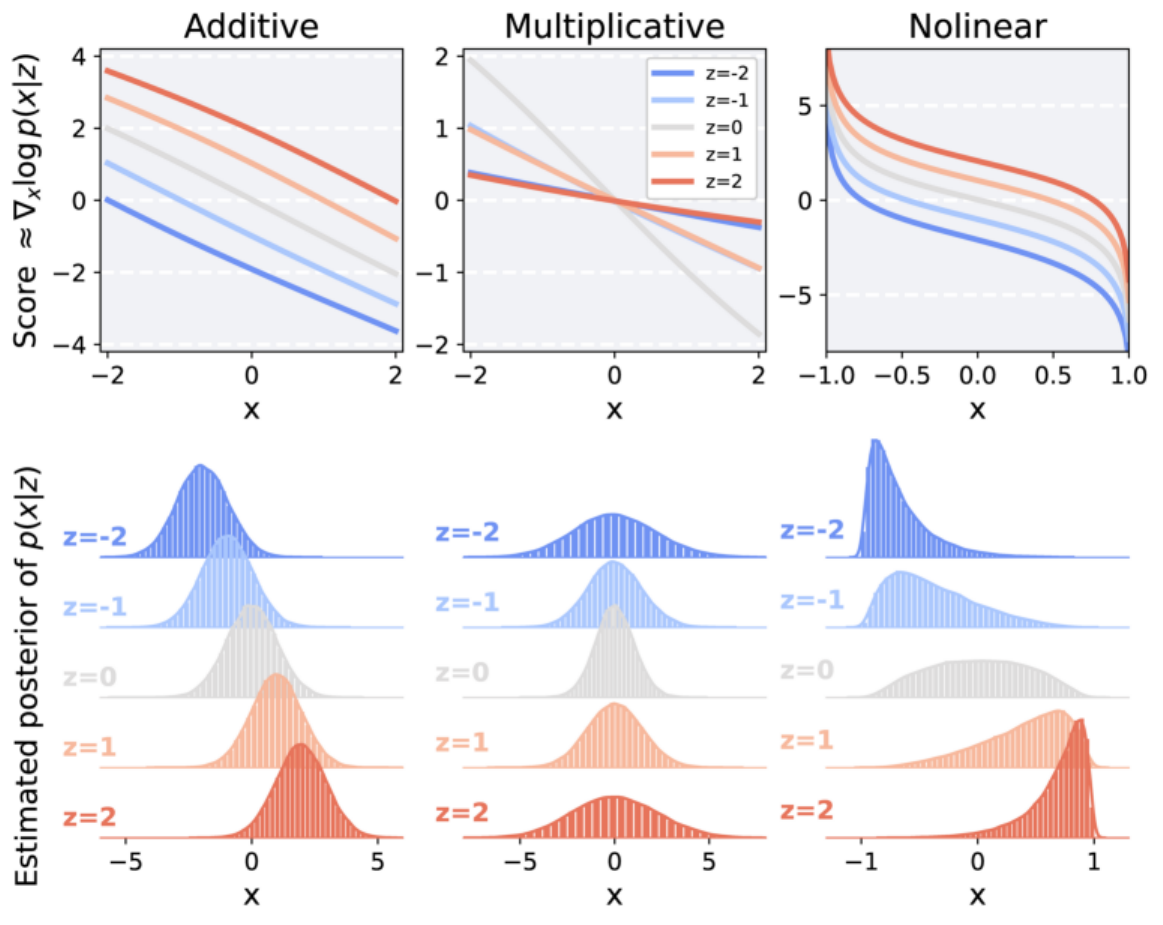}
    \caption{Results of conditional generation experiments. Top: the visualization of the estimated scores. Below: the distributions of generated samples for $z\in \{-2,-1,0,1,2\}$.}
    \label{fig:score_genl}
\end{figure}

\subsection{Results of Distribution Modeling}
We first evaluate the performance of the score-based generative model for modeling conditional distribution. We consider the following three settings. Let $Z,\epsilon_b, \epsilon_x$ be the samples from joint independence one-dimension standard Gaussian, we generate data following three models $X = Z+\epsilon_x$, $X = Z\cdot \epsilon_b+\epsilon_x$, $X = \tanh((Z+\epsilon_x)/2)$, corresponding to the additive, multiplicative, and nonlinear cases respectively. The sample size is set to $50000$. The results of the estimated scores as well as the estimated posterior distributions are illustrated in Fig.~\ref{fig:score_genl}. Our model successfully models the posterior distribution in all cases and accurately estimates the score, especially in high-density regions. This corroborates our theoretical results and thus helps to ensure the validity of our CI test. 

\subsection{CI Testing Results on Synthetic Data}


Here we evaluate our method by a series of experiments, starting with benchmarks from~\citep{scetbon2022asymptotic}, followed by tests in high-dimensional confounder scenarios~\citep{shi2021double}. Additionally, we assess the performance of generative-model-based CI methods under a challenging setting with chain structure~\citep{li2023k}.

\noindent \textbf{Compared methods.} We compare our method with the following 9 existing CI test methods: 
CCIT~\citep{sen2017model}, RCIT~\citep{strobl2019approximate}, FCIT~\citep{chalupka1804fast}, GCM~\citep{shah2020hardness}, KCIT~\citep{zhang2012kernel}, LPCIT~\citep{scetbon2022asymptotic}, GCIT~\citep{bellot2019conditional}, DGCIT~\citep{shi2021double} and NNLSCIT~\citep{li2023nearest}. Software packages of all these methods are freely available online, and default parameters are used if not otherwise specified. For GCIT and DGCIT, we adjust parameters such as learning rate to make the fitting result better for a fair comparison.

\noindent \textbf{Evaluation.} To evaluate test performance, we use (1) \textit{Type I error}, assessing validity by ensuring the error rate is controlled at any significance level $\alpha$; and (2) \textit{testing power}, defined as $1-$Type II error rate, reflecting the ability to detect conditional dependencies.  
 

\noindent \textbf{Results on benchmark datasets.}
Let $Z, \epsilon_x, \epsilon_y$ be samples from the joint independence standard Gaussian, i.e., $Z\sim \mathcal{N}(0_{d_z},\mathbf{I}_{d_z})$. Then the samples are generated as follows:
\begin{equation*}
\begin{split}
X = f_1\left(\alpha\overline{Z}+\beta\epsilon_b+\epsilon_x\right), Y = f_2\left(\alpha\overline{Z}+\beta\epsilon_b+\epsilon_y\right),
\end{split}    
\end{equation*}
where $f_1,f_2\in \{(\cdot), (\cdot)^2, \cos(\cdot), \tanh(\cdot), \exp(-|\cdot|)\}$, referred to as linear, square, cos, tanh, and exp functions, respectively. Here, $\alpha, \beta$ are constants and $\overline{Z}$ is the mean of $Z$ along its dimensions. By using different settings for $\alpha,~\beta$, we simulate varying conditional independence relationships as follows:
\begin{itemize}
\item Case 1: $(\alpha,\beta)=(0.0,0.0)$ for simulating $X\perp \!\!\! \perp Y$.
\item Case 2: $(\alpha,\beta)=(0.0,0.8)$ for simulating $X\not\! \perp \!\!\! \perp Y|Z$.
\item Case 3: $(\alpha,\beta)=(1.0,0.0)$ for simulating $X\perp \!\!\! \perp Y|Z$.
\item Case 4: $(\alpha,\beta)=(1.0,0.8)$ for simulating $X\not\!\perp \!\!\!\perp Y|Z$. 
\end{itemize}
We set the sample size $n=1000$ and the confounder dimension $d_z=10$. Each configuration is repeated $100$ times, and the average result is reported for every function pair $(f_1,f_2)$. 

Fig.~\ref{fig:experiments_gaussian} shows the results, with the four rows corresponding to Cases 1 -- 4, respectively. Our method achieves better Type I error control (with most error values being controlled around 0.05) and test power (close to 1 across various function pairs). In contrast, two other generative model-based CI test methods, GCIT and DGCIT, fail to properly control Type I errors. As shown in the Appendix, visualizations of their generative distributions reveal poor fitting results, which explain their limitations. Similarly, KCIT struggles to control Type I errors due to the challenging setting of $d = 10$. Regression-based methods like GCM and LPCIT, while able to control Type I errors, lack sufficient power for certain function combinations. Overall, our method demonstrates robust performance in both linear and nonlinear scenarios. Visualizations in the Appendix confirm that our generative model can accurately capture conditional distributions, boosting the effectiveness of our method and enhancing its interpretability by providing insights into why the test performs well.

\begin{figure}[h]
    \centering
    \includegraphics[scale=0.19]{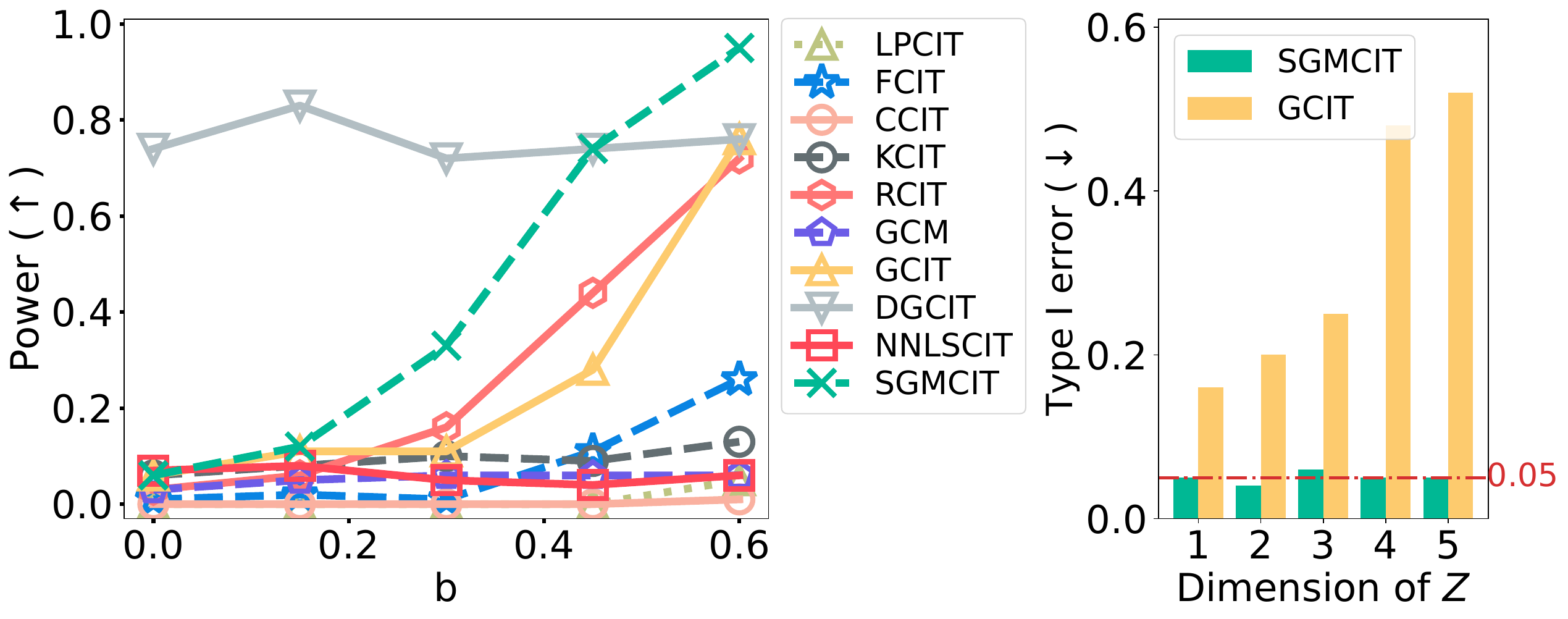}
    \caption{Left: Results in the high-dimensional confounder setting. Right: Results in the chain setting.}
    \label{fig:experiments_highdimensional_chain}
\end{figure}

\noindent \textbf{Results in high-dimensional confounder setting.} To evaluate performance in high-dimensional scenarios, we generate data as
\begin{equation*}
X =\sin(a^T_fZ+\epsilon_f), Y = \cos(a_g^TZ+bX+\epsilon_g),    
\end{equation*}
where $a_f,a_g$ are sampled uniformly from $[0,1]$, and then normalized to have unit $\ell_1$ norm. The noise terms $\epsilon_f,\epsilon_g$ are independent Gaussian samples with mean $0$ and variance $0.25$. Note that $\mathcal{H}_0$ holds when $b$ equals to $0$, otherwise $\mathcal{H}_1$ holds. We set $n=1000$ and $d_z=100$, evaluate test power as $b$ varies from $0$ to $0.6$. 

The average results from $100$ experiments are shown on the left side of Fig.~\ref{fig:experiments_highdimensional_chain}. Our method consistently achieves the highest test power as $b$ increases, while maintaining desirable control of Type I error at $b = 0$. In contrast, although DGCIT achieves high test power, its statistic design leads to uncontrollable Type I error, making the test invalid. GCIT, on the other hand, controls Type I error by reducing the number of test samples, which compromises its power. These findings demonstrate the robustness and adaptability of our method in high-dimensional scenarios.

\noindent \textbf{Results in chain setting.} Here, we evaluate generative model-based CI test methods in a setting that is proved challenging~\citep{li2023nearest}. In the chain setting, $Y\rightarrow Z \rightarrow X$, we compare Type I error control between GCIT and SGMCIT. The data is generated as 
\begin{equation*}
Y\sim \mathcal{N}(1,1), Z = Yu_1+\epsilon_1, X = Z^Tu_2+\epsilon_2,    
\end{equation*}
where $\epsilon_1,\epsilon_2$ are independent Gaussian noise vectors of dimensions $d_z, d_x$ respectively, and the entries of $u_1,u_2$ are uniformly sampled from $[0,0.3]$. We set $n=5000$ and evaluate the Type I error rate with $d_z$ increasing from $1$ to $5$.  

The average results from $100$ experiments are shown on the right of Fig.~\ref{fig:experiments_highdimensional_chain}. Our method stably controls Type I error around 0.05, even as the dimension of $Z$ increases. In contrast, GCIT's Type I error increases with the dimension of $Z$, undermining its reliability. To further investigate this, we perform a goodness-of-fit test and visualize the generating distribution in Fig.~\ref{fig:vischain}. As discussed in Sec.~\ref{subsection_goodnessoffit}, we can see that our method precisely models the generating distribution, enabling precise control of Type I error. These results also highlight the critical role of our goodness-of-fit test in ensuring the validity of the CI testing.

\subsection{CI Testing Results on Real Data}
Following the experimental setup in~\citep{berrett2020conditional}, we evaluate SGMCIT and GCIT on the Capital Bikeshare\footnote{\url{https://s3.amazonaws.com/capitalbikeshare-data/index.html}} dataset. This dataset records details of every ride, including start and end times, locations, and the user type, categorized as either ``Member'' or ``Casual''. In our experiment, we aim to determine whether the ride duration ($X$) is influenced by the user type ($Y$) while controlling variables ($Z$) , which include the starting and ending coordinates (latitude and longitude) and the start time of the ride.
For this experiment, we restrict the dataset to rides taken on weekdays (Monday through Friday) during December 2023, focusing solely on classic bikes. To address outliers, we exclude rides with durations exceeding 7000 seconds. After preprocessing, we randomly select 51,000 samples, which are divided into a training set of 50,000 samples and a test set of 1,000 samples. We evaluate the performance of the goodness-of-fit (GOF) phase as well as the CI testing. For GOF, MMD is employed to compare the distribution difference between the observed and generated joint distributions $(X,Z)$. 

\noindent\textbf{\textit{Results and Analysis.}} The results are presented in Fig.~\ref{fig:real_data}. Both SGMCIT and GCIT conclude that $X\not\! \perp \!\!\! \perp Y|Z$. Since the dataset lacks ground truth, ensuring the reliability of the CI test results becomes critical. The GOF stage plays a crucial role in this regard. Visualizations of the distribution of $X$ and $p$-value of GOF test indicate that SGMCIT achieves significantly better modeling of the conditional distribution, enhancing the reliability of its result. Overall, the GOF stage is proved essential for ensuring the validity of CI testing and providing interpretability in practice. By accurately modeling the conditional distribution, SGMCIT demonstrates clear advantage over existing methods, making it a more desirable choice for real-world scenarios.

\begin{figure}[h]
\centering
\includegraphics[scale=0.41]{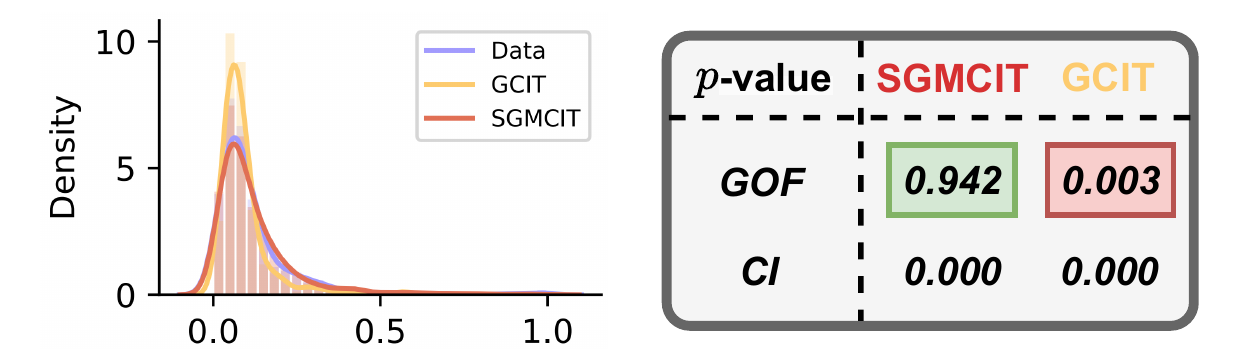}
\caption{CI testing results on real data. Left: Visualization results. Right: $p$-values for GOF test and CI test.}
\label{fig:real_data}
\end{figure}

\section{Conclusion and Future Work}
\label{sec:conclusion}
In this paper, we propose a novel conditional independence testing method based on score-based generative modeling. 
By leveraging a sliced conditional score matching scheme and Langevin dynamics conditional sampling, our method achieves precise Type I error control while maintaining strong testing power in high-dimensional settings. Moreover, the integrated goodness-of-fit validation stage enhances interpretability and reliability in practical scenarios. Both theoretical analyses and extensive experiments on synthetic and real-world datasets demonstrate the effectiveness of our method, establishing it as a new promising direction for generative model-based conditional independence testing.
Future work will focus on accelerating the sampling process.

\section{Post-Acceptance Remarks}\label{sec:post-acceptance remarks}
A concurrent paper by Yang et al.~\citep{yang2025conditional}, recently accepted by AAAI 2025, also introduces a conditional diffusion model for CI testing. We were unaware of this work at the time of submission and thank the reviewer for bringing it to our attention.

Their method is elegant and offers a denoising-based approach. In contrast, our work independently explores a sliced score-matching route, which avoids noise tuning, enables a goodness-of-fit check stage, and provides a different theoretical foundation. Our method also requires fewer hyperparameters and demonstrates better empirical performance across multiple metrics.
We have added a performance comparison with~\citep{yang2025conditional} in the Appendix.


\bibliographystyle{ACM-Reference-Format}
\bibliography{sample-base}


\begin{thebibliography}{60}


\ifx \showCODEN    \undefined \def \showCODEN     #1{\unskip}     \fi
\ifx \showDOI      \undefined \def \showDOI       #1{#1}\fi
\ifx \showISBNx    \undefined \def \showISBNx     #1{\unskip}     \fi
\ifx \showISBNxiii \undefined \def \showISBNxiii  #1{\unskip}     \fi
\ifx \showISSN     \undefined \def \showISSN      #1{\unskip}     \fi
\ifx \showLCCN     \undefined \def \showLCCN      #1{\unskip}     \fi
\ifx \shownote     \undefined \def \shownote      #1{#1}          \fi
\ifx \showarticletitle \undefined \def \showarticletitle #1{#1}   \fi
\ifx \showURL      \undefined \def \showURL       {\relax}        \fi
\providecommand\bibfield[2]{#2}
\providecommand\bibinfo[2]{#2}
\providecommand\natexlab[1]{#1}
\providecommand\showeprint[2][]{arXiv:#2}

\bibitem[Arjovsky et~al\mbox{.}(2017)]%
        {arjovsky2017wasserstein}
\bibfield{author}{\bibinfo{person}{Martin Arjovsky}, \bibinfo{person}{Soumith Chintala}, {and} \bibinfo{person}{L{\'e}on Bottou}.} \bibinfo{year}{2017}\natexlab{}.
\newblock \showarticletitle{Wasserstein generative adversarial networks}. In \bibinfo{booktitle}{\emph{International conference on machine learning}}. PMLR, \bibinfo{publisher}{ACM}, \bibinfo{pages}{214--223}.
\newblock


\bibitem[Azadkia and Chatterjee(2021)]%
        {azadkia2021simple}
\bibfield{author}{\bibinfo{person}{Mona Azadkia} {and} \bibinfo{person}{Sourav Chatterjee}.} \bibinfo{year}{2021}\natexlab{}.
\newblock \showarticletitle{A simple measure of conditional dependence}.
\newblock \bibinfo{journal}{\emph{The Annals of Statistics}} \bibinfo{volume}{49}, \bibinfo{number}{6} (\bibinfo{year}{2021}), \bibinfo{pages}{3070--3102}.
\newblock


\bibitem[Bellot and van~der Schaar(2019)]%
        {bellot2019conditional}
\bibfield{author}{\bibinfo{person}{Alexis Bellot} {and} \bibinfo{person}{Mihaela van~der Schaar}.} \bibinfo{year}{2019}\natexlab{}.
\newblock \showarticletitle{Conditional independence testing using generative adversarial networks}.
\newblock \bibinfo{journal}{\emph{Advances in neural information processing systems}}  \bibinfo{volume}{32} (\bibinfo{year}{2019}), \bibinfo{pages}{2202--2211}.
\newblock


\bibitem[Berrett et~al\mbox{.}(2020)]%
        {berrett2020conditional}
\bibfield{author}{\bibinfo{person}{Thomas~B Berrett}, \bibinfo{person}{Yi Wang}, \bibinfo{person}{Rina~Foygel Barber}, {and} \bibinfo{person}{Richard~J Samworth}.} \bibinfo{year}{2020}\natexlab{}.
\newblock \showarticletitle{The conditional permutation test for independence while controlling for confounders}.
\newblock \bibinfo{journal}{\emph{Journal of the Royal Statistical Society Series B: Statistical Methodology}} \bibinfo{volume}{82}, \bibinfo{number}{1} (\bibinfo{year}{2020}), \bibinfo{pages}{175--197}.
\newblock


\bibitem[Cai et~al\mbox{.}(2022)]%
        {cai2022distribution}
\bibfield{author}{\bibinfo{person}{Zhanrui Cai}, \bibinfo{person}{Runze Li}, {and} \bibinfo{person}{Yaowu Zhang}.} \bibinfo{year}{2022}\natexlab{}.
\newblock \showarticletitle{A distribution free conditional independence test with applications to causal discovery}.
\newblock \bibinfo{journal}{\emph{Journal of Machine Learning Research}} \bibinfo{volume}{23}, \bibinfo{number}{85} (\bibinfo{year}{2022}), \bibinfo{pages}{1--41}.
\newblock


\bibitem[Candes et~al\mbox{.}(2018)]%
        {candes2018panning}
\bibfield{author}{\bibinfo{person}{Emmanuel Candes}, \bibinfo{person}{Yingying Fan}, \bibinfo{person}{Lucas Janson}, {and} \bibinfo{person}{Jinchi Lv}.} \bibinfo{year}{2018}\natexlab{}.
\newblock \showarticletitle{Panning for gold:‘model-X’knockoffs for high dimensional controlled variable selection}.
\newblock \bibinfo{journal}{\emph{Journal of the Royal Statistical Society Series B: Statistical Methodology}} \bibinfo{volume}{80}, \bibinfo{number}{3} (\bibinfo{year}{2018}), \bibinfo{pages}{551--577}.
\newblock


\bibitem[Chalupka et~al\mbox{.}(2018)]%
        {chalupka1804fast}
\bibfield{author}{\bibinfo{person}{Krzysztof Chalupka}, \bibinfo{person}{Pietro Perona}, {and} \bibinfo{person}{Frederick Eberhardt}.} \bibinfo{year}{2018}\natexlab{}.
\newblock \showarticletitle{Fast conditional independence test for vector variables with large sample sizes. arXiv 2018}.
\newblock \bibinfo{journal}{\emph{arXiv preprint arXiv:1804.02747}} (\bibinfo{year}{2018}).
\newblock


\bibitem[Chwialkowski et~al\mbox{.}(2016)]%
        {chwialkowski2016kernel}
\bibfield{author}{\bibinfo{person}{Kacper Chwialkowski}, \bibinfo{person}{Heiko Strathmann}, {and} \bibinfo{person}{Arthur Gretton}.} \bibinfo{year}{2016}\natexlab{}.
\newblock \showarticletitle{A kernel test of goodness of fit}. In \bibinfo{booktitle}{\emph{International conference on machine learning}}. PMLR, \bibinfo{publisher}{ACM}, \bibinfo{pages}{2606--2615}.
\newblock


\bibitem[Daudin(1980)]%
        {daudin1980partial}
\bibfield{author}{\bibinfo{person}{JJ Daudin}.} \bibinfo{year}{1980}\natexlab{}.
\newblock \showarticletitle{Partial association measures and an application to qualitative regression}.
\newblock \bibinfo{journal}{\emph{Biometrika}} \bibinfo{volume}{67}, \bibinfo{number}{3} (\bibinfo{year}{1980}), \bibinfo{pages}{581--590}.
\newblock


\bibitem[Diaconis and Freedman(1980)]%
        {diaconis1980finite}
\bibfield{author}{\bibinfo{person}{Persi Diaconis} {and} \bibinfo{person}{David Freedman}.} \bibinfo{year}{1980}\natexlab{}.
\newblock \showarticletitle{Finite exchangeable sequences}.
\newblock \bibinfo{journal}{\emph{The Annals of Probability}} \bibinfo{volume}{8}, \bibinfo{number}{6} (\bibinfo{year}{1980}), \bibinfo{pages}{745--764}.
\newblock


\bibitem[Doran et~al\mbox{.}(2014)]%
        {doran2014permutation}
\bibfield{author}{\bibinfo{person}{Gary Doran}, \bibinfo{person}{Krikamol Muandet}, \bibinfo{person}{Kun Zhang}, {and} \bibinfo{person}{Bernhard Sch{\"o}lkopf}.} \bibinfo{year}{2014}\natexlab{}.
\newblock \showarticletitle{A Permutation-Based Kernel Conditional Independence Test.}. In \bibinfo{booktitle}{\emph{UAI}}. \bibinfo{publisher}{AUAI}, \bibinfo{pages}{132--141}.
\newblock


\bibitem[Gretton et~al\mbox{.}(2012)]%
        {gretton2012kernel}
\bibfield{author}{\bibinfo{person}{Arthur Gretton}, \bibinfo{person}{Karsten~M Borgwardt}, \bibinfo{person}{Malte~J Rasch}, \bibinfo{person}{Bernhard Sch{\"o}lkopf}, {and} \bibinfo{person}{Alexander Smola}.} \bibinfo{year}{2012}\natexlab{}.
\newblock \showarticletitle{A kernel two-sample test}.
\newblock \bibinfo{journal}{\emph{The Journal of Machine Learning Research}} \bibinfo{volume}{13}, \bibinfo{number}{1} (\bibinfo{year}{2012}), \bibinfo{pages}{723--773}.
\newblock


\bibitem[Gretton et~al\mbox{.}(2009)]%
        {gretton2009fast}
\bibfield{author}{\bibinfo{person}{Arthur Gretton}, \bibinfo{person}{Kenji Fukumizu}, \bibinfo{person}{Zaid Harchaoui}, {and} \bibinfo{person}{Bharath~K Sriperumbudur}.} \bibinfo{year}{2009}\natexlab{}.
\newblock \showarticletitle{A fast, consistent kernel two-sample test}.
\newblock \bibinfo{journal}{\emph{Advances in neural information processing systems}}  \bibinfo{volume}{22} (\bibinfo{year}{2009}), \bibinfo{pages}{673--681}.
\newblock


\bibitem[Gretton et~al\mbox{.}(2007)]%
        {gretton2007kernel}
\bibfield{author}{\bibinfo{person}{Arthur Gretton}, \bibinfo{person}{Kenji Fukumizu}, \bibinfo{person}{Choon Teo}, \bibinfo{person}{Le Song}, \bibinfo{person}{Bernhard Sch{\"o}lkopf}, {and} \bibinfo{person}{Alex Smola}.} \bibinfo{year}{2007}\natexlab{}.
\newblock \showarticletitle{A kernel statistical test of independence}.
\newblock \bibinfo{journal}{\emph{Advances in neural information processing systems}}  \bibinfo{volume}{20} (\bibinfo{year}{2007}), \bibinfo{pages}{585--592}.
\newblock


\bibitem[Ho et~al\mbox{.}(2020)]%
        {ho2020denoising}
\bibfield{author}{\bibinfo{person}{Jonathan Ho}, \bibinfo{person}{Ajay Jain}, {and} \bibinfo{person}{Pieter Abbeel}.} \bibinfo{year}{2020}\natexlab{}.
\newblock \showarticletitle{Denoising diffusion probabilistic models}.
\newblock \bibinfo{journal}{\emph{Advances in neural information processing systems}}  \bibinfo{volume}{33} (\bibinfo{year}{2020}), \bibinfo{pages}{6840--6851}.
\newblock


\bibitem[Huang et~al\mbox{.}(2022)]%
        {huang2022kernel}
\bibfield{author}{\bibinfo{person}{Zhen Huang}, \bibinfo{person}{Nabarun Deb}, {and} \bibinfo{person}{Bodhisattva Sen}.} \bibinfo{year}{2022}\natexlab{}.
\newblock \showarticletitle{Kernel partial correlation coefficient—a measure of conditional dependence}.
\newblock \bibinfo{journal}{\emph{The Journal of Machine Learning Research}} \bibinfo{volume}{23}, \bibinfo{number}{1} (\bibinfo{year}{2022}), \bibinfo{pages}{9699--9756}.
\newblock


\bibitem[Huber and Melly(2015)]%
        {huber2015test}
\bibfield{author}{\bibinfo{person}{Martin Huber} {and} \bibinfo{person}{Blaise Melly}.} \bibinfo{year}{2015}\natexlab{}.
\newblock \showarticletitle{A test of the conditional independence assumption in sample selection models}.
\newblock \bibinfo{journal}{\emph{Journal of Applied Econometrics}} \bibinfo{volume}{30}, \bibinfo{number}{7} (\bibinfo{year}{2015}), \bibinfo{pages}{1144--1168}.
\newblock


\bibitem[Hyv{\"a}rinen and Dayan(2005)]%
        {hyvarinen2005estimation}
\bibfield{author}{\bibinfo{person}{Aapo Hyv{\"a}rinen} {and} \bibinfo{person}{Peter Dayan}.} \bibinfo{year}{2005}\natexlab{}.
\newblock \showarticletitle{Estimation of non-normalized statistical models by score matching.}
\newblock \bibinfo{journal}{\emph{Journal of Machine Learning Research}} \bibinfo{volume}{6}, \bibinfo{number}{4} (\bibinfo{year}{2005}), \bibinfo{pages}{695--709}.
\newblock


\bibitem[Javanmard and Mehrabi(2021)]%
        {javanmard2021pearson}
\bibfield{author}{\bibinfo{person}{Adel Javanmard} {and} \bibinfo{person}{Mohammad Mehrabi}.} \bibinfo{year}{2021}\natexlab{}.
\newblock \showarticletitle{Pearson chi-squared conditional randomization test}.
\newblock \bibinfo{journal}{\emph{arXiv preprint arXiv:2111.00027}} (\bibinfo{year}{2021}).
\newblock


\bibitem[Jitkrittum et~al\mbox{.}(2020)]%
        {jitkrittum2020testing}
\bibfield{author}{\bibinfo{person}{Wittawat Jitkrittum}, \bibinfo{person}{Heishiro Kanagawa}, {and} \bibinfo{person}{Bernhard Sch{\"o}lkopf}.} \bibinfo{year}{2020}\natexlab{}.
\newblock \showarticletitle{Testing goodness of fit of conditional density models with kernels}. In \bibinfo{booktitle}{\emph{Conference on Uncertainty in Artificial Intelligence}}. PMLR, \bibinfo{publisher}{AUAI}, \bibinfo{pages}{221--230}.
\newblock


\bibitem[Jitkrittum et~al\mbox{.}(2017)]%
        {jitkrittum2017linear}
\bibfield{author}{\bibinfo{person}{Wittawat Jitkrittum}, \bibinfo{person}{Wenkai Xu}, \bibinfo{person}{Zolt{\'a}n Szab{\'o}}, \bibinfo{person}{Kenji Fukumizu}, {and} \bibinfo{person}{Arthur Gretton}.} \bibinfo{year}{2017}\natexlab{}.
\newblock \showarticletitle{A linear-time kernel goodness-of-fit test}.
\newblock \bibinfo{journal}{\emph{Advances in Neural Information Processing Systems}}  \bibinfo{volume}{30} (\bibinfo{year}{2017}), \bibinfo{pages}{261--270}.
\newblock


\bibitem[Kong et~al\mbox{.}(2020)]%
        {kong2020diffwave}
\bibfield{author}{\bibinfo{person}{Zhifeng Kong}, \bibinfo{person}{Wei Ping}, \bibinfo{person}{Jiaji Huang}, \bibinfo{person}{Kexin Zhao}, {and} \bibinfo{person}{Bryan Catanzaro}.} \bibinfo{year}{2020}\natexlab{}.
\newblock \showarticletitle{Diffwave: A versatile diffusion model for audio synthesis}.
\newblock \bibinfo{journal}{\emph{arXiv preprint arXiv:2009.09761}} (\bibinfo{year}{2020}).
\newblock


\bibitem[Lee et~al\mbox{.}(2022)]%
        {lee2022convergence}
\bibfield{author}{\bibinfo{person}{Holden Lee}, \bibinfo{person}{Jianfeng Lu}, {and} \bibinfo{person}{Yixin Tan}.} \bibinfo{year}{2022}\natexlab{}.
\newblock \showarticletitle{Convergence for score-based generative modeling with polynomial complexity}.
\newblock \bibinfo{journal}{\emph{Advances in Neural Information Processing Systems}}  \bibinfo{volume}{35} (\bibinfo{year}{2022}), \bibinfo{pages}{22870--22882}.
\newblock


\bibitem[Li and Fan(2020)]%
        {li2020nonparametric}
\bibfield{author}{\bibinfo{person}{Chun Li} {and} \bibinfo{person}{Xiaodan Fan}.} \bibinfo{year}{2020}\natexlab{}.
\newblock \showarticletitle{On nonparametric conditional independence tests for continuous variables}.
\newblock \bibinfo{journal}{\emph{Wiley Interdisciplinary Reviews: Computational Statistics}} \bibinfo{volume}{12}, \bibinfo{number}{3} (\bibinfo{year}{2020}), \bibinfo{pages}{e1489}.
\newblock


\bibitem[Li et~al\mbox{.}(2023a)]%
        {li2023nearest}
\bibfield{author}{\bibinfo{person}{Shuai Li}, \bibinfo{person}{Ziqi Chen}, \bibinfo{person}{Hongtu Zhu}, \bibinfo{person}{Christina~Dan Wang}, {and} \bibinfo{person}{Wang Wen}.} \bibinfo{year}{2023}\natexlab{a}.
\newblock \showarticletitle{Nearest-neighbor sampling based conditional independence testing}. In \bibinfo{booktitle}{\emph{Proceedings of the AAAI Conference on Artificial Intelligence}}, Vol.~\bibinfo{volume}{37}. \bibinfo{publisher}{AAAI}, \bibinfo{pages}{8631--8639}.
\newblock


\bibitem[Li et~al\mbox{.}(2023b)]%
        {li2023k}
\bibfield{author}{\bibinfo{person}{Shuai Li}, \bibinfo{person}{Yingjie Zhang}, \bibinfo{person}{Hongtu Zhu}, \bibinfo{person}{Christina~Dan Wang}, \bibinfo{person}{Hai Shu}, \bibinfo{person}{Ziqi Chen}, \bibinfo{person}{Zhuoran Sun}, {and} \bibinfo{person}{Yanfeng Yang}.} \bibinfo{year}{2023}\natexlab{b}.
\newblock \showarticletitle{K-Nearest-Neighbor Local Sampling Based Conditional Independence Testing}. In \bibinfo{booktitle}{\emph{Thirty-seventh Conference on Neural Information Processing Systems}}. \bibinfo{publisher}{MIT Press}, \bibinfo{pages}{23321--23344}.
\newblock


\bibitem[Lopez-Paz et~al\mbox{.}(2013)]%
        {lopez2013randomized}
\bibfield{author}{\bibinfo{person}{David Lopez-Paz}, \bibinfo{person}{Philipp Hennig}, {and} \bibinfo{person}{Bernhard Sch{\"o}lkopf}.} \bibinfo{year}{2013}\natexlab{}.
\newblock \showarticletitle{The randomized dependence coefficient}.
\newblock \bibinfo{journal}{\emph{Advances in neural information processing systems}}  \bibinfo{volume}{26} (\bibinfo{year}{2013}), \bibinfo{pages}{1--9}.
\newblock


\bibitem[Mittag(2018)]%
        {mittag2018nonparametric}
\bibfield{author}{\bibinfo{person}{Nikolas Mittag}.} \bibinfo{year}{2018}\natexlab{}.
\newblock \bibinfo{booktitle}{\emph{A Nonparametric k-Sample Test of Conditional Independence}}.
\newblock \bibinfo{type}{{T}echnical {R}eport}. \bibinfo{institution}{Working paper of CERGE-EI, Prague, Czech Republic}.
\newblock


\bibitem[Park and Muandet(2020)]%
        {park2020measure}
\bibfield{author}{\bibinfo{person}{Junhyung Park} {and} \bibinfo{person}{Krikamol Muandet}.} \bibinfo{year}{2020}\natexlab{}.
\newblock \showarticletitle{A measure-theoretic approach to kernel conditional mean embeddings}.
\newblock \bibinfo{journal}{\emph{Advances in neural information processing systems}}  \bibinfo{volume}{33} (\bibinfo{year}{2020}), \bibinfo{pages}{21247--21259}.
\newblock


\bibitem[Pogodin et~al\mbox{.}(2022)]%
        {pogodin2022efficient}
\bibfield{author}{\bibinfo{person}{Roman Pogodin}, \bibinfo{person}{Namrata Deka}, \bibinfo{person}{Yazhe Li}, \bibinfo{person}{Danica~J Sutherland}, \bibinfo{person}{Victor Veitch}, {and} \bibinfo{person}{Arthur Gretton}.} \bibinfo{year}{2022}\natexlab{}.
\newblock \showarticletitle{Efficient conditionally invariant representation learning}.
\newblock \bibinfo{journal}{\emph{arXiv preprint arXiv:2212.08645}} (\bibinfo{year}{2022}).
\newblock


\bibitem[Polo et~al\mbox{.}(2023)]%
        {polo2023conditional}
\bibfield{author}{\bibinfo{person}{Felipe~Maia Polo}, \bibinfo{person}{Yuekai Sun}, {and} \bibinfo{person}{Moulinath Banerjee}.} \bibinfo{year}{2023}\natexlab{}.
\newblock \showarticletitle{Conditional independence testing under model misspecification}.
\newblock \bibinfo{journal}{\emph{arXiv preprint arXiv:2307.02520}} (\bibinfo{year}{2023}).
\newblock


\bibitem[Ramesh et~al\mbox{.}(2022)]%
        {ramesh2022hierarchical}
\bibfield{author}{\bibinfo{person}{Aditya Ramesh}, \bibinfo{person}{Prafulla Dhariwal}, \bibinfo{person}{Alex Nichol}, \bibinfo{person}{Casey Chu}, {and} \bibinfo{person}{Mark Chen}.} \bibinfo{year}{2022}\natexlab{}.
\newblock \showarticletitle{Hierarchical text-conditional image generation with clip latents}.
\newblock \bibinfo{journal}{\emph{arXiv preprint arXiv:2204.06125}} \bibinfo{volume}{1}, \bibinfo{number}{2} (\bibinfo{year}{2022}), \bibinfo{pages}{3}.
\newblock


\bibitem[Ren et~al\mbox{.}(2024a)]%
        {renefficiently}
\bibfield{author}{\bibinfo{person}{Yixin Ren}, \bibinfo{person}{Yewei Xia}, \bibinfo{person}{Hao Zhang}, \bibinfo{person}{Jihong Guan}, {and} \bibinfo{person}{Shuigeng Zhou}.} \bibinfo{year}{2024}\natexlab{a}.
\newblock \showarticletitle{Efficiently Learning Significant Fourier Feature Pairs for Statistical Independence Testing}. In \bibinfo{booktitle}{\emph{The Thirty-eighth Annual Conference on Neural Information Processing Systems}}. \bibinfo{publisher}{MIT Press}.
\newblock


\bibitem[Ren et~al\mbox{.}(2024b)]%
        {ren2024learning}
\bibfield{author}{\bibinfo{person}{Yixin Ren}, \bibinfo{person}{Yewei Xia}, \bibinfo{person}{Hao Zhang}, \bibinfo{person}{Jihong Guan}, {and} \bibinfo{person}{Shuigeng Zhou}.} \bibinfo{year}{2024}\natexlab{b}.
\newblock \showarticletitle{Learning Adaptive Kernels for Statistical Independence Tests}. In \bibinfo{booktitle}{\emph{International Conference on Artificial Intelligence and Statistics}}. PMLR, \bibinfo{publisher}{JMLR}, \bibinfo{pages}{2494--2502}.
\newblock


\bibitem[Ren et~al\mbox{.}(2023)]%
        {ren2023multi}
\bibfield{author}{\bibinfo{person}{Yixin Ren}, \bibinfo{person}{Hao Zhang}, \bibinfo{person}{Yewei Xia}, \bibinfo{person}{Jihong Guan}, {and} \bibinfo{person}{Shuigeng Zhou}.} \bibinfo{year}{2023}\natexlab{}.
\newblock \showarticletitle{Multi-level wavelet mapping correlation for statistical dependence measurement: methodology and performance}. In \bibinfo{booktitle}{\emph{Proceedings of the AAAI Conference on Artificial Intelligence}}, Vol.~\bibinfo{volume}{37}. \bibinfo{publisher}{AAAI}, \bibinfo{pages}{6499--6506}.
\newblock


\bibitem[Roberts and Tweedie(1996)]%
        {roberts1996exponential}
\bibfield{author}{\bibinfo{person}{Gareth~O Roberts} {and} \bibinfo{person}{Richard~L Tweedie}.} \bibinfo{year}{1996}\natexlab{}.
\newblock \showarticletitle{Exponential convergence of Langevin distributions and their discrete approximations}.
\newblock \bibinfo{journal}{\emph{Bernoulli}} \bibinfo{volume}{2}, \bibinfo{number}{3} (\bibinfo{year}{1996}), \bibinfo{pages}{341--363}.
\newblock


\bibitem[Runge(2018)]%
        {runge2018conditional}
\bibfield{author}{\bibinfo{person}{Jakob Runge}.} \bibinfo{year}{2018}\natexlab{}.
\newblock \showarticletitle{Conditional independence testing based on a nearest-neighbor estimator of conditional mutual information}. In \bibinfo{booktitle}{\emph{International Conference on Artificial Intelligence and Statistics}}. PMLR, \bibinfo{publisher}{JMLR}, \bibinfo{pages}{938--947}.
\newblock


\bibitem[Scetbon et~al\mbox{.}(2022)]%
        {scetbon2022asymptotic}
\bibfield{author}{\bibinfo{person}{Meyer Scetbon}, \bibinfo{person}{Laurent Meunier}, {and} \bibinfo{person}{Yaniv Romano}.} \bibinfo{year}{2022}\natexlab{}.
\newblock \showarticletitle{An asymptotic test for conditional independence using analytic kernel embeddings}. In \bibinfo{booktitle}{\emph{International Conference on Machine Learning}}. PMLR, \bibinfo{publisher}{ACM}, \bibinfo{pages}{19328--19346}.
\newblock


\bibitem[Sen et~al\mbox{.}(2017)]%
        {sen2017model}
\bibfield{author}{\bibinfo{person}{Rajat Sen}, \bibinfo{person}{Ananda~Theertha Suresh}, \bibinfo{person}{Karthikeyan Shanmugam}, \bibinfo{person}{Alexandros~G Dimakis}, {and} \bibinfo{person}{Sanjay Shakkottai}.} \bibinfo{year}{2017}\natexlab{}.
\newblock \showarticletitle{Model-powered conditional independence test}.
\newblock \bibinfo{journal}{\emph{Advances in neural information processing systems}}  \bibinfo{volume}{30} (\bibinfo{year}{2017}), \bibinfo{pages}{2955--2965}.
\newblock


\bibitem[Shah and Peters(2020)]%
        {shah2020hardness}
\bibfield{author}{\bibinfo{person}{Rajen~D Shah} {and} \bibinfo{person}{Jonas Peters}.} \bibinfo{year}{2020}\natexlab{}.
\newblock \showarticletitle{The hardness of conditional independence testing and the generalised covariance measure}.
\newblock  (\bibinfo{year}{2020}).
\newblock


\bibitem[Sheng and Sriperumbudur(2019)]%
        {sheng2019distance}
\bibfield{author}{\bibinfo{person}{Tianhong Sheng} {and} \bibinfo{person}{Bharath~K Sriperumbudur}.} \bibinfo{year}{2019}\natexlab{}.
\newblock \showarticletitle{On distance and kernel measures of conditional independence}.
\newblock \bibinfo{journal}{\emph{arXiv preprint arXiv:1912.01103}} (\bibinfo{year}{2019}).
\newblock


\bibitem[Shi et~al\mbox{.}(2021b)]%
        {shi2021double}
\bibfield{author}{\bibinfo{person}{Chengchun Shi}, \bibinfo{person}{Tianlin Xu}, \bibinfo{person}{Wicher Bergsma}, {and} \bibinfo{person}{Lexin Li}.} \bibinfo{year}{2021}\natexlab{b}.
\newblock \showarticletitle{Double generative adversarial networks for conditional independence testing}.
\newblock \bibinfo{journal}{\emph{The Journal of Machine Learning Research}} \bibinfo{volume}{22}, \bibinfo{number}{1} (\bibinfo{year}{2021}), \bibinfo{pages}{13029--13060}.
\newblock


\bibitem[Shi et~al\mbox{.}(2021a)]%
        {shi2021azadkia}
\bibfield{author}{\bibinfo{person}{Hongjian Shi}, \bibinfo{person}{Mathias Drton}, {and} \bibinfo{person}{Fang Han}.} \bibinfo{year}{2021}\natexlab{a}.
\newblock \showarticletitle{On Azadkia-Chatterjee's conditional dependence coefficient}.
\newblock \bibinfo{journal}{\emph{arXiv preprint arXiv:2108.06827}} (\bibinfo{year}{2021}).
\newblock


\bibitem[Sohl-Dickstein et~al\mbox{.}(2015)]%
        {sohl2015deep}
\bibfield{author}{\bibinfo{person}{Jascha Sohl-Dickstein}, \bibinfo{person}{Eric Weiss}, \bibinfo{person}{Niru Maheswaranathan}, {and} \bibinfo{person}{Surya Ganguli}.} \bibinfo{year}{2015}\natexlab{}.
\newblock \showarticletitle{Deep unsupervised learning using nonequilibrium thermodynamics}. In \bibinfo{booktitle}{\emph{International conference on machine learning}}. PMLR, \bibinfo{publisher}{ACM}, \bibinfo{pages}{2256--2265}.
\newblock


\bibitem[Song and Dhariwal(2023)]%
        {song2023improved}
\bibfield{author}{\bibinfo{person}{Yang Song} {and} \bibinfo{person}{Prafulla Dhariwal}.} \bibinfo{year}{2023}\natexlab{}.
\newblock \showarticletitle{Improved techniques for training consistency models}.
\newblock \bibinfo{journal}{\emph{arXiv preprint arXiv:2310.14189}} (\bibinfo{year}{2023}).
\newblock


\bibitem[Song et~al\mbox{.}(2023)]%
        {song2023consistency}
\bibfield{author}{\bibinfo{person}{Yang Song}, \bibinfo{person}{Prafulla Dhariwal}, \bibinfo{person}{Mark Chen}, {and} \bibinfo{person}{Ilya Sutskever}.} \bibinfo{year}{2023}\natexlab{}.
\newblock \showarticletitle{Consistency models}.
\newblock \bibinfo{journal}{\emph{arXiv preprint arXiv:2303.01469}} (\bibinfo{year}{2023}).
\newblock


\bibitem[Song et~al\mbox{.}(2021)]%
        {song2021maximum}
\bibfield{author}{\bibinfo{person}{Yang Song}, \bibinfo{person}{Conor Durkan}, \bibinfo{person}{Iain Murray}, {and} \bibinfo{person}{Stefano Ermon}.} \bibinfo{year}{2021}\natexlab{}.
\newblock \showarticletitle{Maximum likelihood training of score-based diffusion models}.
\newblock \bibinfo{journal}{\emph{Advances in Neural Information Processing Systems}}  \bibinfo{volume}{34} (\bibinfo{year}{2021}), \bibinfo{pages}{1415--1428}.
\newblock


\bibitem[Song and Ermon(2019)]%
        {song2019generative}
\bibfield{author}{\bibinfo{person}{Yang Song} {and} \bibinfo{person}{Stefano Ermon}.} \bibinfo{year}{2019}\natexlab{}.
\newblock \showarticletitle{Generative modeling by estimating gradients of the data distribution}.
\newblock \bibinfo{journal}{\emph{Advances in neural information processing systems}}  \bibinfo{volume}{32} (\bibinfo{year}{2019}), \bibinfo{pages}{11918--11930}.
\newblock


\bibitem[Song et~al\mbox{.}(2020a)]%
        {song2020sliced}
\bibfield{author}{\bibinfo{person}{Yang Song}, \bibinfo{person}{Sahaj Garg}, \bibinfo{person}{Jiaxin Shi}, {and} \bibinfo{person}{Stefano Ermon}.} \bibinfo{year}{2020}\natexlab{a}.
\newblock \showarticletitle{Sliced score matching: A scalable approach to density and score estimation}. In \bibinfo{booktitle}{\emph{Uncertainty in Artificial Intelligence}}. PMLR, \bibinfo{pages}{574--584}.
\newblock


\bibitem[Song et~al\mbox{.}(2020b)]%
        {song2020score}
\bibfield{author}{\bibinfo{person}{Yang Song}, \bibinfo{person}{Jascha Sohl-Dickstein}, \bibinfo{person}{Diederik~P Kingma}, \bibinfo{person}{Abhishek Kumar}, \bibinfo{person}{Stefano Ermon}, {and} \bibinfo{person}{Ben Poole}.} \bibinfo{year}{2020}\natexlab{b}.
\newblock \showarticletitle{Score-based generative modeling through stochastic differential equations}.
\newblock \bibinfo{journal}{\emph{arXiv preprint arXiv:2011.13456}} (\bibinfo{year}{2020}).
\newblock


\bibitem[Spirtes et~al\mbox{.}(1993)]%
        {Spirtes1993CausationPA}
\bibfield{author}{\bibinfo{person}{Peter Spirtes}, \bibinfo{person}{Clark Glymour}, {and} \bibinfo{person}{Richard Scheines}.} \bibinfo{year}{1993}\natexlab{}.
\newblock \showarticletitle{Causation, prediction, and search}. \bibinfo{publisher}{MIT press}.
\newblock


\bibitem[Strobl et~al\mbox{.}(2019)]%
        {strobl2019approximate}
\bibfield{author}{\bibinfo{person}{Eric~V Strobl}, \bibinfo{person}{Kun Zhang}, {and} \bibinfo{person}{Shyam Visweswaran}.} \bibinfo{year}{2019}\natexlab{}.
\newblock \showarticletitle{Approximate kernel-based conditional independence tests for fast non-parametric causal discovery}.
\newblock \bibinfo{journal}{\emph{Journal of Causal Inference}} \bibinfo{volume}{7}, \bibinfo{number}{1} (\bibinfo{year}{2019}), \bibinfo{pages}{20180017}.
\newblock


\bibitem[Vincent(2011)]%
        {vincent2011connection}
\bibfield{author}{\bibinfo{person}{Pascal Vincent}.} \bibinfo{year}{2011}\natexlab{}.
\newblock \showarticletitle{A connection between score matching and denoising autoencoders}.
\newblock \bibinfo{journal}{\emph{Neural computation}} \bibinfo{volume}{23}, \bibinfo{number}{7} (\bibinfo{year}{2011}), \bibinfo{pages}{1661--1674}.
\newblock


\bibitem[Wang et~al\mbox{.}(2015)]%
        {wang2015conditional}
\bibfield{author}{\bibinfo{person}{Xueqin Wang}, \bibinfo{person}{Wenliang Pan}, \bibinfo{person}{Wenhao Hu}, \bibinfo{person}{Yuan Tian}, {and} \bibinfo{person}{Heping Zhang}.} \bibinfo{year}{2015}\natexlab{}.
\newblock \showarticletitle{Conditional distance correlation}.
\newblock \bibinfo{journal}{\emph{J. Amer. Statist. Assoc.}} \bibinfo{volume}{110}, \bibinfo{number}{512} (\bibinfo{year}{2015}), \bibinfo{pages}{1726--1734}.
\newblock


\bibitem[Warren(2021)]%
        {warren2021wasserstein}
\bibfield{author}{\bibinfo{person}{Andrew Warren}.} \bibinfo{year}{2021}\natexlab{}.
\newblock \showarticletitle{Wasserstein conditional independence testing}.
\newblock \bibinfo{journal}{\emph{arXiv preprint arXiv:2107.14184}} (\bibinfo{year}{2021}).
\newblock


\bibitem[Wenliang et~al\mbox{.}(2019)]%
        {wenliang2019learning}
\bibfield{author}{\bibinfo{person}{Li Wenliang}, \bibinfo{person}{Danica~J Sutherland}, \bibinfo{person}{Heiko Strathmann}, {and} \bibinfo{person}{Arthur Gretton}.} \bibinfo{year}{2019}\natexlab{}.
\newblock \showarticletitle{Learning deep kernels for exponential family densities}. In \bibinfo{booktitle}{\emph{International Conference on Machine Learning}}. PMLR, \bibinfo{publisher}{ACM}, \bibinfo{pages}{6737--6746}.
\newblock


\bibitem[Yang et~al\mbox{.}(2025)]%
        {yang2025conditional}
\bibfield{author}{\bibinfo{person}{Yanfeng Yang}, \bibinfo{person}{Shuai Li}, \bibinfo{person}{Yingjie Zhang}, \bibinfo{person}{Zhuoran Sun}, \bibinfo{person}{Hai Shu}, \bibinfo{person}{Ziqi Chen}, {and} \bibinfo{person}{Renming Zhang}.} \bibinfo{year}{2025}\natexlab{}.
\newblock \showarticletitle{Conditional diffusion models based conditional independence testing}. In \bibinfo{booktitle}{\emph{Proceedings of the AAAI Conference on Artificial Intelligence}}, Vol.~\bibinfo{volume}{39}. \bibinfo{pages}{22020--22028}.
\newblock


\bibitem[Zhang et~al\mbox{.}(2018)]%
        {zhang2018measuring}
\bibfield{author}{\bibinfo{person}{Hao Zhang}, \bibinfo{person}{Shuigeng Zhou}, {and} \bibinfo{person}{Jihong Guan}.} \bibinfo{year}{2018}\natexlab{}.
\newblock \showarticletitle{Measuring conditional independence by independent residuals: theoretical results and application in causal discovery}. In \bibinfo{booktitle}{\emph{Proceedings of the AAAI Conference on Artificial Intelligence}}, Vol.~\bibinfo{volume}{32}. \bibinfo{publisher}{AAAI}, \bibinfo{pages}{2029--2036}.
\newblock


\bibitem[Zhang et~al\mbox{.}(2012)]%
        {zhang2012kernel}
\bibfield{author}{\bibinfo{person}{Kun Zhang}, \bibinfo{person}{Jonas Peters}, \bibinfo{person}{Dominik Janzing}, {and} \bibinfo{person}{Bernhard Sch{\"o}lkopf}.} \bibinfo{year}{2012}\natexlab{}.
\newblock \showarticletitle{Kernel-based conditional independence test and application in causal discovery}.
\newblock \bibinfo{journal}{\emph{arXiv preprint arXiv:1202.3775}} (\bibinfo{year}{2012}).
\newblock


\bibitem[Zhang et~al\mbox{.}(2017)]%
        {zhang2017feature}
\bibfield{author}{\bibinfo{person}{Qinyi Zhang}, \bibinfo{person}{Sarah Filippi}, \bibinfo{person}{Seth Flaxman}, {and} \bibinfo{person}{Dino Sejdinovic}.} \bibinfo{year}{2017}\natexlab{}.
\newblock \showarticletitle{Feature-to-feature regression for a two-step conditional independence test}.
\newblock  (\bibinfo{year}{2017}).
\newblock


\end{thebibliography}

\appendix
$~$

\noindent {{The Appendix is organized as follows:}}
\begin{itemize}[leftmargin=5mm]
\item Section~\ref{a0}: List of Symbols and Notations.
\item Section~\ref{a1}: Proof of Conditional Sliced Score Matching.
\item Section~\ref{a2}: Proof of Langevin Dynamics Conditional Sampling.
\item Section~\ref{a3}: Proof of Exchangeablility.
\item Section~\ref{a4}: Proof of Type I error Bound. 
\item Section~\ref{a5}: Implementation Details.
\item Section~\ref{a6}: More Experiments Results.
\end{itemize}

\section{List of Symbols and Notations}
\label{a0}
\begin{table}[htbp]
\begin{center}
\small
\begin{tabular}{lll}
\toprule
\rule{0pt}{1pt} $\mathcal{O}, o$ & & big, small O notion\\
\rule{0pt}{1pt} $\mathcal{O}_p, o_p$ & & big, small O notion in probability\\
\rule{0pt}{1pt} $i.i.d.$ & & independent and identically distributed\\
\rule{0pt}{1pt} $r.v.s$ & & random variables\\
\rule{0pt}{1pt} $\mathbb{R}$ & & the set of real numbers\\
\rule{0pt}{1pt} $\mathcal{B}(\mathbb{R})$ & & Borel $\sigma$-algebra on $\mathbb{R}$\\
\rule{0pt}{1pt} $\mathbb{P}_X$ & & marginal distribution of $X$\\
\rule{0pt}{1pt} $\mathbb{P}_{XY}$ & & joint distribution of $X$, $Y$\\
\rule{0pt}{1pt} $p(x)$ & & probability density function of $X$\\
\rule{0pt}{1pt} $p(x|z)$ & & conditional probability density function of $X|Z$\\
\rule{0pt}{1pt} $\mathbb{E}[X]$ & & expectation of $X$\\
\rule{0pt}{1pt} $\text{Var}(X)$ & & variance of $X$\\
\rule{0pt}{1pt} $X\perp \!\!\! \perp Y$ & & r.v.s $X$, $Y$ are independent\\
\rule{0pt}{1pt} $X\not\! \perp\!\!\!\perp Y|Z$ & & r.v.s $X$, $Y$ are not independent, given $Z$.\\
\rule{0pt}{1pt} $\text{Tr}(\cdot)$ & & the trace of a square matrix\\
\rule{0pt}{1pt} $\mathbf{1}(\cdot)$ & &  the indicator function \\
\rule{0pt}{1pt} 
$[B]$ & & the simplified notation, defined as $[B]:= 1,2,...,B$\\
\rule{0pt}{1pt} 
$\times $ & & the product symbol of topological space \\
\rule{0pt}{1pt} 
$s(x,z;\theta) $ & & the conditional score function parameterized by $\theta$\\
\rule{0pt}{1pt} 
$\nabla_x s(x,z;\theta)$ & & the gradient of score function w.r.t. $x$\\
\rule{0pt}{1pt} 
$\xrightarrow[]{\text{d}}, \xrightarrow[]{\text{p}}$ & & convergence in distribution, in probability \\
\rule{0pt}{1pt} 
${\buildrel d \over =}$ & & equality in distribution \\
\bottomrule
\end{tabular}
\end{center}
\label{tab:TableOfNotationForMyResearch}
\end{table}
\newpage
\section{Proof of Conditional Sliced Score Matching}
\label{a1}
In this section, we extend the conclusions of~\citep{song2020sliced} to our conditional score function. We begin by summarizing the most commonly used notations.  Let the dataset be $\mathcal{D}=\{(x_i, y_i, z_i)\}^n_{i=1}$ and the conditional probability be  $p(x|z)$. The model score function, denoted as  $s(x,z;\theta)$, corresponding to $p(x,z;\theta)$ where $\theta$ is restricted to a parameter space $\Theta$. The goal of $s(x,z;\theta)$ is to approximate the true score function $\nabla_x\log p(x, z)$. The Hessian of $\log p(x,z)$ w.r.t. $x$ is represented as $\nabla_x s(x,z;\theta)$. Additionally, we introduce $v$ as a random vector of the same dimension as $x$, referred to as the projection vector, with $p_v$ denoting its distribution.

\renewcommand{\theassumption}{B.\arabic{assumption}}
\setcounter{assumption}{0} 
\begin{assumption}[Regularity of conditional scores]
For any $z$, $s(x,z;\theta)$ and $\nabla_x\log p(x|z)$ are both differentiable w.r.t. $x$. 
Additionally, we assume that they satisfy $\mathbb{E}_{x\sim p(x|z)}[\|s(x,z;\theta)\|_2^2] < \infty$ and $\mathbb{E}_{x\sim p(x|z)}[\|\nabla_x\log p(x|z)\|_2^2] < \infty$.
\end{assumption}

\begin{assumption}[Regularity of projection vectors] 
The projection vectors satisfy 
$\mathbb{E}_{v\sim p_v}[\|v\|_2^2] < \infty$, and 
$\mathbb{E}_{v\sim p_v}[vv^T] \succ 0$.
\end{assumption}

\begin{assumption}[Boundary conditions]
For any $z$, for all $\theta \in \Theta$, the score satisfy $\lim_{\|x\| \to \infty} s(x,z;\theta) p(x|z) = 0.$
\end{assumption}

\begin{assumption} [Identifiability]
The model family $\{p(x,z; \theta) \mid \theta \in \Theta\}$ is well-specified, i.e., $p(x,z) = p(x,z; \theta^*)$. 
Furthermore, $p(x,z; \theta) \neq p(x,z; \theta^*)$ whenever $\theta \neq \theta^*$.
\end{assumption}

\begin{assumption}[Positiveness]
The probability density function satisfies $p(x,z; \theta) > 0, \ \forall \theta \in \Theta \ \text{and} \ \forall (x,z).$ 
\end{assumption}

\begin{lemma} 
\label{app_lemma_of_loss_nomargin}
Assume $s(x,z;\theta)$, $\nabla_x\log p(x, z)$ and $p_v$ satisfy some regularity conditions (Assumption B.1, Assumption B.2). Under proper boundary conditions (Assumption B.3), we have 
\begin{equation}
\label{b1}
\begin{split}
\mathcal{L}_{\theta}&(z;p_v) =  \frac{1}{2}\mathbb{E}_{\substack{v\sim p_v\\x\sim p(x|z)}}\left\{\bigl[ v^Ts(x,z;\theta)-v^T\nabla_x \log p(x|z)\bigl]_2^2\right\}\\
&= \mathbb{E}_{v\sim p_v}\mathbb{E}_{x\sim p(x|z)} \left\{v^T\nabla_x s(x,z;\theta)v+\frac{1}{2}[v^Ts(x,z;\theta)]^2\right\} + C(z),
\end{split}
\end{equation}
where $C(z)$ is a constant w.r.t. $\theta$.
\end{lemma}
\begin{proof} For a fix $z$, , our proof follows the approach of~\citep{song2020sliced}. To enhance readability, we recount the key details. Since expectations are bounded under Assumptions B.1 and B.2, we expand $\mathcal{L}_{\theta}(z;p_v)$ as
\begin{equation}
\begin{split}
\label{b2}
\mathcal{L}_{\theta}&(z;p_v)  = \frac{1}{2}\mathbb{E}_{\substack{v\sim p_v\\x\sim p(x|z)}}\left\{\bigl[ v^Ts(x,z;\theta)-v^T\nabla_x \log p(x|z)\bigl]_2^2\right\}
\\&= \frac{1}{2} \mathbb{E}_{v\sim p_v} \mathbb{E}_{x\sim p(x|z)} 
\Bigl\{ [v^T s(x, z; \theta)]^2 + [v^T \nabla_x 
\log p(x|z)\bigl]^2 \\ & \ \ \ \ \ \ \ \ \ \ \ \ \ \ \ \ \ \ \ \ \ \ \ \ \ \ \ \ \ \ \ \ \ \ \ \ \ \ \ \ \ \ \ -2 \bigl[v^T s(x, z; \theta)\bigl]\bigl[v^T\nabla_x \log p(x|z)\bigl] \Bigr\}
\\&=\mathbb{E}_{v\sim p_v} \mathbb{E}_{x\sim p(x|z)} 
\Bigl\{ - [v^T s(x, z; \theta)][v^T \nabla_x \log p(x|z)\bigl]  \\& \ \ \ \ \ \ \ \ \ \ \ \ \ \ \ \ \ \ \ \ \ \ \ \ \ \ \ \ \ \ \ \ \ \ \ \ \ \ \ \ \ \ \ \ \ \ \ \ +\frac{1}{2} [v^T s(x, z; \theta)]^2 \Bigr\} + C(z), 
\end{split}
\end{equation}
where we have absorbed the term related to $\nabla_x \log p(x|z)$ into $C(z)$ since it does not depend on $\theta$. Next, we show that 
\begin{equation}
\begin{split}
\label{b3}
-\mathbb{E}_{v\sim p_v} \mathbb{E}_{x\sim p(x|z)} \left\{
\bigl[v^T s(x, z; \theta)\bigl]\bigl[v^T \nabla_x \log p(x|z)\bigl]\right\} \\= \mathbb{E}_{v\sim p_v} \mathbb{E}_{x\sim p(x|z)} 
\Bigl[v^T \nabla_x s(x, z; \theta) v\Bigl]. 
\end{split}
\end{equation}
This can be shown by first calculating that
\begin{equation}
\begin{split}
\label{b4}
&  \mathbb{E}_{v\sim p_v} \mathbb{E}_{x\sim p(x|z)} \left\{
\bigl[v^T s(x, z; \theta)\bigl]\bigl[v^T \nabla_x \log p(x|z)\bigl]\right\}
\\&= \mathbb{E}_{v\sim p_v} \int p(x|z)\bigl[v^T s(x, z; \theta)\bigl]\bigl[v^T \nabla_x \log p(x|z)\bigl] dx
\\&= \mathbb{E}_{v\sim p_v} \int \bigl[v^T s(x, z; \theta)\bigl] \bigl[v^T \nabla_x p(x|z)\bigl] dx
\\&= \mathbb{E}_{v\sim p_v}\sum_{i=1}^{d_x} \int \bigl[v^T  s(x, z; \theta) \bigl] v_i \frac{\partial p(x|z)}{\partial x_i} dx, 
\end{split}
\end{equation}
where recall that $x \in \mathbb{R}^{d_x}$. Then, applying multivariate integration by parts, we have
\begin{equation*}
\begin{split}
\ &  \bigg| \mathbb{E}_{v\sim p_v} \sum_{i=1}^{d_x} \int \bigl[v^T s(x, z; \theta) \bigl] v_i \frac{\partial p(x|z)}{\partial x_i} dx 
\\ &  \ \ \ \ \ \ \ \ \ \ \ \ \ \ \ \  + \mathbb{E}_{v\sim p_v} \sum_{i=1}^{d_x}\int \bigl[v_ip(x|z)\bigl]  v^T\frac{\partial s(x, z; \theta)}{\partial x_i} dx \bigg| \\
&= \bigg|\mathbb{E}_{v\sim p_v}  \sum_{i=1}^{d_x} \Bigl\{\lim_{x_i \to +\infty} \big[ v^T s(x, z; \theta) \big] v_i p(x|z) 
\\ &  \ \ \ \ \ \ \ \ \ \ \ \ \ \ \ \ \ \ \ \ \ \ \ \ \ \ \ \ \ \ \ \ \ \ \ \ \ - \lim_{x_i \to -\infty} \big[v^T s(x, z; \theta) \big] v_i p(x|z) \Bigl\} \bigg| \\
&\leq \sum_{i=1}^{d_x} \lim_{x_i \to +\infty} \sum_{j=1}^{d_x} \mathbb{E}_{v\sim p_v}  |v_i v_j| | s_j(x, z; \theta)  p(x|z) |
\\&  \ \ \ \ \ \ \ \ \ \ \ \ \ \ \ \ \ \ \ \ \ \ \ \ + \sum_{i=1}^{d_x} \lim_{x_i \to -\infty} \sum_{j=1}^{d_x} \mathbb{E}_{v\sim p_v}  |v_i v_j| | s_j(x, z; \theta) p(x|z) |\\
&\overset{\text{(i)}}{\leq} \sum_{i=1}^D \lim_{x_i \to \infty} \sum_{j=1}^D 
\sqrt{\mathbb{E}_{v\sim p_v} v_i^2 \mathbb{E}_{v\sim p_v}  v_j^2 } \cdot|s_j(x, z; \theta) p(x|z)| 
\\ & \ \ \ \ + \sum_{i=1}^{d_x} \lim_{x_i \to -\infty} \sum_{j=1}^{d_x} 
\sqrt{\mathbb{E}_{v\sim p_v} v_i^2  \mathbb{E}_{v\sim p_v} v_j^2 } \cdot|s_j(x, z; \theta) p(x|z)| \overset{\text{(ii)}}{=} 0,
\end{split}
\end{equation*}
where $s_j(x, z; \theta)$ denotes the $j$-th component of $s(x, z; \theta)$. In the above derivation, $(i)$ is due to Cauchy-Schwarz inequality and $(ii)$ is from the {Assumption B.2 and B.3} that $\mathbb{E}_{v\sim p_v}[\|v\|^2] < \infty$ and $\lim_{\|x\| \to \infty} s(x,z;\theta) p(x|z) = 0$. As a result, for Eq.~(\ref{b4}), we have
\begin{equation}
\begin{split}
&\mathbb{E}_{v\sim p_v} \sum_{i=1}^{d_x} \int \bigl[v^T s(x, z; \theta) \bigl] v_i \frac{\partial p(x|z)}{\partial x_i} dx 
\\= &-\mathbb{E}_{v\sim p_v} \sum_{i=1}^{d_x}\int \bigl[v_ip(x|z)\bigl]  v^T\frac{\partial s(x, z; \theta)}{\partial x_i} dx \\
= &-\mathbb{E}_{v\sim p_v} \int p(x|z) \bigl[v^T \nabla_x s(x, z; \theta) v\bigl] dx \\ = &-\mathbb{E}_{v\sim p_v} \mathbb{E}_{x\sim p(x|z)} 
\bigl[v^T \nabla_x s(x, z; \theta) v\bigl],
\end{split}
\end{equation}
which proves Eq.~(\ref{b3}) and the proof is completed.
\end{proof}
Lemma~\ref{app_lemma_of_loss_nomargin} derives $\mathcal{L}_{\theta}(z;p_v)$ while avoiding terms related to $\nabla_x \log p(x|z)$. The overall objective is then obtained by marginalization, i.e., $\mathcal{L}_{\theta}(p_v) := \mathbb{E}_{z\sim p(z)}[\mathcal{L}_{\theta}(z;p_v)]$. Next, in Lemma~\ref{app_identifiable}, we establish key properties of $\mathcal{L}_{\theta}(p_v)$ showing that the solution satisfying $\mathcal{L}_{\theta}(p_v) = 0$ corresponds to the optimal parameter $\theta^*$.

\begin{lemma} 
\label{app_identifiable}
Assume the model family is well-specified and identifiable (Assumption B.4). Assume further that the densities are all positive (Assumption B.5). When $p_v$ satisfies Assumption B.2, we have
\begin{equation}
\begin{split}
\mathcal{L}_{\theta}(p_v) := \mathbb{E}_{z\sim p(z)}[\mathcal{L}_{\theta}(z;p_v)] = 0 \Leftrightarrow \theta = \theta^*.
\end{split}
\end{equation}
\end{lemma}
\begin{proof} e obtain the proof by extending the results of~\citep{song2020sliced} to the conditional case. Under Assumptions B.4 and B.5, for all $(x,z)$, $p(x,z) = p(x,z; \theta^*) > 0$. Recall the definition of the loss function: 
\begin{equation}
\mathcal{L}_{\theta}(p_v) := \frac{1}{2}\mathbb{E}_{\substack{v\sim p_v\\(x,z)\sim p(x,z)}}\left\{\bigl[ v^Ts(x,z;\theta)-v^T\nabla_x \log p(x|z)\bigl]_2^2\right\}.   
\end{equation}
Hence, $\mathcal{L}_{\theta}(p_v) = 0$ implies for all $(x,z)$, \\ $\mathbb{E}_{v\sim p_v}\left\{\bigl[ v^Ts(x,z;\theta)-v^T\nabla_x \log p(x|z)\bigl]_2^2\right\} = 0$. Further, 
\begin{equation}
\label{eq14}
\begin{split}
\mathbb{E}_{v\sim p_v}& \left\{\bigl[ v^Ts(x,z;\theta)-v^T\nabla_x \log p(x|z)\bigl]_2^2\right\} = 0 \\
\Leftrightarrow ~&\mathbb{E}_{v\sim p_v} \Bigl\{ v^T \big[s(x, z; \theta) - \nabla_x \log p(x|z)\big] \\& \ \ \ \ \ \ \ \ \ \ \cdot \big[ s(x, z; \theta) - \nabla_x \log p(x|z) \big]^T v \Bigr\}= 0 \\
\Leftrightarrow ~&\big[s(x, z; \theta) - \nabla_x \log p(x|z)\big]^T \cdot \mathbb{E}_{v\sim p_v} \big[v v^T\big] \\ & \ \ \ \ \ \ \ \ \ \ \cdot \big[s(x, z; \theta) - \nabla_x \log p(x|z)\big]= 0 .
\end{split}
\end{equation}
By the {Assumption B.2}, $\mathbb{E}_{v\sim p_v} \big[v v^T\big]$ is positive definite. Therefore, Eq.~(\ref{eq14}) implies that for any $(x,z)$, the equality $s(x, z; \theta) = \nabla_x \log p(x|z)$ holds. Integrating both sides w.r.t. $x$, yields:  
\begin{equation}
s(x, z; \theta) = \nabla_x \log p(x|z) \Leftrightarrow p(x,z;\theta) =  p(x,z) + C_0,  
\end{equation}
where note that $p(x,z;\theta)$ is the probability density function corresponding to the score function $s(x, z; \theta)$ and $C_0$ is a constant. Since both $p(x,z;\theta)$ and $p(x,z)$ are normalized probability density functions, thus $C_0 = 0$. Therefore, by Assumption B.4, we conclude that $\theta = \theta^*$. The remain proof for the right-to-left direction is trivial.
\end{proof}
Thus, the process of finding the optimal parameters is equivalent to minimizing the loss objective $\mathcal{L}_{\theta}(p_v)$. 
In the main paper, we omit constant terms that are independent of $\theta$, i.e., the final optimization objective is given by
\begin{equation}
\mathcal{J}_{\theta} =\mathbb{E}_{\substack{v\sim p_v\\(x,z)\sim p(x,z)}} \left\{v^T\nabla_x s(x,z;\theta)v+\frac{1}{2}[v^Ts(x,z;\theta)]^2\right\}.
\end{equation}
Note that $\theta^* = \mathop{\arg \min}\limits_{\theta\in\Theta} {\mathcal{J}}_{\theta}$. Further, in practice, a finite sample approximation is used, which is expressed as:
\begin{equation}
\widehat{\mathcal{J}}_{\theta}=\frac{1}{nm}\sum_{i,j}^{n,m} \Bigl\{v_{ij}^T\nabla_{x_i} s(x_i,z_i;\theta) v_{ij}^T+\frac{1}{2}[v_{ij}^Ts(x_i,z_i;\theta)]^2\Bigl\}.    
\end{equation}
Next, we prove the consistency of $\hat{\theta}_{n,m} := \mathop{\arg \min}\limits_{\theta\in\Theta} \widehat{\mathcal{J}}_{\theta}$. The following additional assumptions are needed.

\begin{assumption} The parameter space $\Theta$ is compact.
\end{assumption}

\begin{assumption}[Lipschitz continuity]
Both the term $\nabla_x s(x, z; \theta)$ and $s(x, z; \theta) s(x, z; \theta)^T$ are Lipschitz continuous in terms of Frobenious norm, i.e., for all $ \theta_1, \theta_2 \in \Theta$, $||\nabla_x s(x, z; \theta_1) - \nabla_x s(x, z; \theta_2)||_F \leq $\\ $ L_1(x,z)||\theta_1 - \theta_2||_2$, $||s(x, z; \theta_1) s(x, z; \theta_1)^T - s(x, z; \theta_2) s(x, z; \theta_2)^T||_F $\\$\leq L_2(x,z)||\theta_1 - \theta_2||_2$. In addition, we require the Lipschitz constant satisfy $\mathbb{E}_{(x,z)}[L_1^2(x,z)] < \infty$ and $\mathbb{E}_{(x,z)}[L_2^2(x,z)] < \infty$.
\end{assumption}

\begin{assumption} [Bounded moments of projection vectors]
The moments of projection vectors satisfy 
$\mathbb{E}_{v\sim p_v}[||v v^T||_F^2] < \infty$.
\end{assumption}

\begin{lemma}[Uniform convergence of the expected error]\label{app:lemma:1}
Under the Assumption B.6-B.8, we have
\begin{align}
    \mathbb{E}_{\substack{v\sim p_v\\(x,z)\sim p(x,z)}}\bigg[\sup_{\theta \in \Theta} \big|\widehat{\mathcal{J}}_{\theta} - \mathcal{J}_{\theta} \big| \bigg] \leq \mathcal{O}\biggl(\operatorname{diam}(\Theta) \sqrt{\frac{d_\Theta}{n}} \biggl),
\end{align}
where $\operatorname{diam}(\cdot)$ denotes the diameter and $d_{\Theta}$ is the dimension of $\Theta$.
\end{lemma}
\begin{proof}
The proof follows by modifying the proof of Lemma 3 in~\citep{song2020sliced}. We begin by defining $f(\theta; x, z, v) := v^T \nabla_x s(x, z; \theta) v + \frac{1}{2}(v^T s(x, z; \theta))^2$. Under Assumptions B.7 and B.8, we can show that $f(\theta; x, v)$ is Lipschitz continuous with constant $L(x,z, v)$ satisfying $\mathbb{E}_{(x,z)\sim p(x,z), v\sim p_v}[L^2(x,z, v)] < \infty$. Using a symmetrization trick and chaining technique, we can then derive the bound
\begin{equation}
\begin{split}
 &\mathbb{E}_{\substack{v\sim p_v\\(x,z)\sim p(x,z)}}\bigg[\sup_{\theta \in \Theta} \big|\widehat{\mathcal{J}}_{\theta} - \mathcal{J}_{\theta} \big| \bigg] \\ & \leq \mathcal{O}(1)\operatorname{diam}(\Theta) \sqrt{\frac{d_\Theta}{n}} \sqrt{\mathbb{E}_{(x,z)\sim p(x,z), v\sim p_v}[L^2(x,z, v)]}, 
\end{split}
\end{equation}
which completes the proof.
\end{proof}

\begin{theorem}[Consistency]
\label{appendix_consistency_theorem}
Under the Assumption B.1-B.8,  
$\hat{\theta}_{n,m}$ is consistent, meaning that $\hat{\theta}_{n, m} \overset{p}{\to} \theta^*$
as $n\to \infty$.
\end{theorem}
\begin{proof}
The proof follows directly from the arguments in Theorem 2 of~\citep{song2020sliced}. Specifically, the objective $\mathcal{J}_{\theta}$ exhibits similar continuous properties in the compact parameter space $\Theta$ as outlined in Lemma~\ref{app:lemma:1}.    
\end{proof}

\section{Proof of Langevin Dynamics Conditional Sampling (LDCS)}
\label{a2}
We start by recalling the process of Langevin dynamics conditional sampling (LDCS). Let the step size be $h$, the total time be $T$. For a fixed $z$, the sampling process that iteratively updates $x_{kh}$ as 
\begin{equation}
\label{langevin_dynamics_conditional_sampling_app}
x_{(k+1)h} = x_{kh} + h \cdot s(x_{kh},z; \hat{\theta}_{n,m}) + \sqrt{2h}\cdot \xi_{kh}, 
\end{equation}
where $\xi_{kh}\sim\mathcal{N}(0,I_{d_x})$ and $x_0\sim p_0(x|z)$ for initializing. We take $p_0(x|z) = \mathcal{N}(0,I_{d_x})$ in practice. In the following, we analyze the error control of the generated distribution. The proof in this section requires the following additional assumptions.
\renewcommand{\theassumption}{C.\arabic{assumption}}
\setcounter{assumption}{0} 
\begin{assumption} [Smoothness] For any $z$, $\log p(x|z)$ is continuously differentiable ($C^1$) w.r.t. $x$ and is $L_z$-smooth w.r.t. $x$, meaning that the conditional score function $\nabla_x \log p(x|z)$ is $L_z$-Lipschitz. Additionally, we assume $L_z \geq 1$ for all $z$.
\end{assumption} 

\begin{assumption} [Log-Sobolev inequality constraints] For any $z$, we assume that $p(x|z)$ satisfies a log-Sobolev inequality with constant $C_{z;\text{LS}}$. Furthermore, we assume $C_{z;\text{LS}} \geq 1$ for all $z$.
\end{assumption}




\begin{assumption}[$L^2$-accurate] For any $z$, the error in the conditional score estimate is bounded in ${L}^2$, i.e.
\begin{equation}
\begin{split}
&\|\nabla_x \log p(x|z) - s(x,z;\hat{\theta}_{n,m})\|^2_{{L}^2(p(x|z))} \\&:= \mathbb{E}_{x\sim p(x|z)}[\|\nabla_x \log p(x|z) - s(x,z;\hat{\theta}_{n,m})\|^2] \leq \varepsilon_z^2.   
\end{split}
\end{equation}
\end{assumption}

    

\noindent\textbf{Remark.} Note that Assumption C.3 is closely related to the result of score matching in the previous step. Recall that, according to Theorem~\ref{appendix_consistency_theorem}, we have shown that $\hat{\theta}_{n, m} \overset{p}{\to} \theta^*$ as $n\to \infty$. By the continuous mapping theorem, this implies that for all $x, z$, $s(x,z;\hat{\theta}_{n, m}) \overset{p}{\to} s(x,z;\theta^*)$ as $n\to \infty$. Therefore, Assumption C.3 holds asymptotically, but for convenience in stating the following theorem, we include it as an assumption at this stage. 

The error bound between the sampled distribution and the data distribution is provided by the following theorem. The proof can be derived by modifying the proof of Theorem 1.2 in~\citep{lee2022convergence}, assuming an $L^2$-accurate conditional score function estimate. The primary difference is that the constants in the analysis now depend on $z$.

\begin{theorem}[LDCS with $L^2$-accurate score estimate] Under Assumptions C.1–C.3, consider an accuracy requirement in total variation (TV) distance: $0 < \varepsilon_{z;\text{TV}} < 1$. Suppose further that the initial distribution satisfies $d_{\chi^2}\{p_0(x|z) \| p(x|z)\} \leq K_z^2$. Then if
\begin{equation}
\begin{split}
\varepsilon_z \leq \frac{\varepsilon_{z;\text{TV}}^4}{174080\sqrt{5}d_x L_z^2 C_{z;\text{LS}}^{5/2} \max\{\ln(2 K_z / \varepsilon_{z;\text{TV}}^2),  2K_z\}},
\end{split}
\end{equation}
then running LDCS with score estimate $s(x,z;\hat{\theta}_{n, m})$, step size $h =\frac{\varepsilon_{z;\text{TV}}^2}{2720 d_x L_z^2 C_{z;\text{LS}}}$, and total time 
$T=  4C_{z;\text{LS}} \ln\left(\frac{2 K_z}{\varepsilon_{z;\text{TV}}^2}\right)$,
yields a distribution $p_T(x|z)$ satisfying the error bound 
\begin{equation}
d_\text{TV}\{p_T(x|z), p(x|z)\} \leq 2 \varepsilon_{z;\text{TV}}.    
\end{equation}
\end{theorem} 
\begin{proof}
For fixed $z$, the proof can be obtained by modifying the proof of Theorem 1.2 in~\citep{lee2022convergence}.      
\end{proof}
To simplify the notation, we introduce universal constants that hold for all $z$. Specifically, we define $K^2:=\sup_z\{K_z^2\}$, the Lipschitz constant $L = \sup_z\{L_z\}$ and the constant for log-Sobolev inequality as $C_{\text{LS}}:=\sup_z\{C_{z;\text{LS}}\}$. Then if we aim to control the accuracy for all $z$ within $2\varepsilon_{TV}$, we require that for all $z$,  
\begin{equation}
\begin{split}
\label{appendix_condition}
\varepsilon_z &\leq \frac{\varepsilon_{\text{TV}}^4}{174080\sqrt{5}d_x L^2 C_{\text{LS}}^{5/2} \max\{\ln(2 K / \varepsilon_{\text{TV}}^2),  2K\}} =: \varepsilon_c , \\h &=\frac{\varepsilon_{\text{TV}}^2}{2720 d_x L^2 C_{\text{LS}}}, T=  4C_{\text{LS}} \ln\left(\frac{2 K}{\varepsilon_{\text{TV}}^2}\right).   
\end{split}
\end{equation}

\begin{theorem}[Error bound of conditional distribution] \label{appenddix Error bound of conditional distribution}
Under {Assumptions C.1 and C.2}, running LDCS with the score estimate $s(x,z;\hat{\theta}_{n, m})$, with an appropriate step size $h$, and time $T$, then for any $z$, results in a conditional distribution $p_T(x|z)$ such that the error guarantee that $d_\text{TV}\{p_T(x|z), p(x|z)\} = o_p(1)$.
\end{theorem} 
\begin{proof}
By Theorem~\ref{appendix_consistency_theorem}, we have shown that $\hat{\theta}_{n, m} \overset{p}{\to} \theta^*$ as $n\to \infty$. Applying the continuous mapping theorem, we obtain that for all $x, z$, $s(x,z;\hat{\theta}_{n, m}) \overset{p}{\to} s(x,z;\theta^*)$ as $n\to \infty$. In other equivalent form, for all $x,z$, for any $\epsilon>0$, we have
\begin{equation}
\lim_{n\to \infty} \mathbb{P}\left(\|s(x,z;\hat{\theta}_{n, m}) - s(x,z;\theta^*)\|\leq \epsilon\right)  = 1.   
\end{equation}
Additionally, since under the condition given by Eq.~(\ref{appendix_condition}), the event $d_\text{TV}\{p_T(x|z), p(x|z)\}\leq 2\varepsilon_{\text{TV}}$ will happen, yield:
\begin{equation}
\begin{split}
 &\mathbb{P}\left(d_\text{TV}\{p_T(x|z), p(x|z)\}\leq 2\varepsilon_{\text{TV}} \right) \\\geq ~&\mathbb{P}\left(\mathbb{E}_{x\sim p(x|z)}[\|\nabla_x \log p(x|z) - s(x,z;\hat{\theta}_{n,m})\|^2] \leq \varepsilon_c^2 \right)\\
 = ~&\mathbb{P}\left(\mathbb{E}_{x\sim p(x|z)}[\|s(x,z;\theta^*) - s(x,z;\hat{\theta}_{n,m})\|^2] \leq \varepsilon_c^2 \right)\\
\geq ~&\mathbb{P}\left(\|s(x,z;\theta^*) - s(x,z;\hat{\theta}_{n,m})\|^2 \leq \varepsilon_c^2 \right).
\end{split}
\end{equation}
By setting $\epsilon = \varepsilon_c$, and taking the limit on both sides, we obtain
\begin{equation}
\begin{split}
&\lim_{n\to \infty}\mathbb{P}\left(d_\text{TV}\{p_T(x|z), p(x|z)\}\leq 2\varepsilon_{\text{TV}} \right)\\ &\geq \lim_{n\to \infty} \mathbb{P}\left(\|s(x,z;\theta^*) - s(x,z;\hat{\theta}_{n,m})\|^2 \leq \varepsilon_c^2 \right) = 1  
\end{split}
\end{equation}
for any given $\varepsilon_{\text{TV}}\in (0,1)$, thus by definition of ``converge in distribution'' notion, we have $d_\text{TV}\{p_T(x|z), p(x|z)\} = o_p(1)$.
\end{proof}

\section{Proof of Exchangeablility}
\label{a3}
In this section, we prove the exchangeability property, which ensures the validity of $p$-values under certain assumptions.

\begin{proposition}[Exchangeablility of triples] 
\label{app_exchangeabliliy_random_variable}
Let ${\buildrel d \over =}$ denotes equality in distribution. Then under $\mathcal{H}_0$, and further assuming that for all $b\in [B]$, $(X^{(b)},Y,Z)~{\buildrel d \over =}~(X,Y,Z)$, the resulting random sequence of triples $(X^{(b)},Y,Z)_{b=0}^B$ is exchangeable. Recall that in the main paper, we use $(X^{(0)}, Y, Z)$ to represent the observed triple $(X,Y,Z)$.
\end{proposition}
\begin{proof}
A sequence of random variables is said to be exchangeable if its distribution is invariant under variable permutations. By the "representation theorem"~\citep{diaconis1980finite} for exchangeable sequences of random variables, that show that every sequence of conditionally $i.i.d.$ random variables can be considered as a sequence of exchangeable random variables. Recall the process of our generative model, we start from $i.i.d.$ sequence of init random variables $x_0^{(b)}$, which are iteratively updated as:
\begin{equation}
x_{(k+1)h}^{(b)} = x_{kh}^{(b)} +h \cdot s(x_{kh}^{(b)},Z; \theta) + \sqrt{2h}\cdot\xi_{kh},      
\end{equation}
where $\xi_{kh}\sim\mathcal{N}(0,I_{d_x})$. Note that for each step $t=kh$, we can represent the generated process of $x_T^{(b)}$ as 
\begin{equation}
\label{app_generative_process}
x_T^{(b)} = \phi_{T}(\cdots \phi_t(\cdots \phi_1(x_0^{(b)}; Z, \xi_{h}); Z, \xi_{kh}); Z, \xi_{T}),    
\end{equation}
where $\phi_t(x_{(k-1)h}^{(b)}; Z, \xi_{kh}) = x_{(k-1)h}^{(b)} + h \cdot s(x_{(k-1)h}^{(b)},Z; \theta) + \sqrt{2h}\cdot\xi_{kh}$. By the construction of Eq.~(\ref{app_generative_process}), since the score function $s(\cdot,z;\theta)$ is measurable and the additional noise $\xi_{kh}$ and $Z$ are independent to the $x_0^{(b)}$, the resulting random sequence of random variables $(X^{(b)},Y,Z)_{b=1}^B$ is exchangeable according to the "representation theorem", thus completes the proof.
\end{proof}
Next, we show the exchangeability of the statistic derived from the random sequence, as shown in the following corollary. 
\begin{corollary}[Exchangeablility of statistics] 
\label{app_exchangeabliliy_statistics}
Let ${\buildrel d \over =}$ denotes equality in distribution. Then under $\mathcal{H}_0$, and further assume that for all $b\in [B]$, $(X^{(b)},Y,Z)~{\buildrel d \over =}~(X,Y,Z)$, the resulting random sequence of statistics $[{\rho(X^{(b)}, Y, Z)}]_{b=0}^B$ is exchangeable. 
\end{corollary}
\begin{proof}
By Proposition~\ref{app_exchangeabliliy_random_variable}, the resulting random sequence of triples $(X^{(b)},Y,Z)_{b=0}^B$ is exchangeable. Since $\rho$ is a measurable function, the sequence of statistics $[{\rho(X^{(b)}, Y, Z)}]_{b=0}^B$ is also exchangeable by the "representation theorem". 
\end{proof}
Given that the sequence of statistics is exchangeable, we now demonstrate that the $p$-value obtained by
\begin{equation}
\label{app_pvalue}
p\text{-value}= \frac{1+\sum_{b=1}^B\mathbf{1}\{\rho(X^{(b)},Y,Z)\geq \rho(X,Y,Z)\}}{1+B}      
\end{equation}
is valid, as shown in the following proposition. 
\begin{proposition}[Valid $p$-value] 
\label{app_valid_pvalue}
Under $\mathcal{H}_0$, and assuming that for all $b\in [B]$, $(X^{(b)},Y,Z)~{\buildrel d \over =}~(X,Y,Z)$, let the random sequence of statistics be $[{\rho(X^{(b)}, Y, Z)}]_{b=0}^B$. Then the $p$-value given by Eq.~(\ref{app_pvalue}) is valid, i.e., 
\begin{equation}
\mathbb{P}(p\text{-value}\leq \alpha|\mathcal{H}_0)\leq \alpha, ~~\text{for any given }\alpha\in (0,1).    
\end{equation}
\end{proposition}
\begin{proof}
For simplify, we also write $\mathbb{P}(\cdot|\mathcal{H}_0)$ as $\mathbb{P}_{\mathcal{H}_0}$. For any given $\alpha\in (0,1)$, we have
\begin{equation}
\begin{split}
 &\mathbb{P}_{\mathcal{H}_0}(p\text{-value}\leq \alpha)  \\&= \mathbb{P}_{\mathcal{H}_0}\left(\frac{1+\sum_{b=1}^B\mathbf{1}\{\rho(X^{(b)},Y,Z)\geq \rho(X,Y,Z)\}}{1+B}\leq \alpha\right)\\
 &\leq  \mathbb{P}_{\mathcal{H}_0}\left({\sum_{b=1}^B\mathbf{1}\{\rho(X^{(b)},Y,Z)\geq \rho(X,Y,Z)\}}\leq \lfloor\alpha({1+B})\rfloor\right).
\end{split}
\end{equation}
Since the sequence $[{\rho(X^{(b)}, Y, Z)}]_{b=0}^B$ is exchangeable, by the property of order statistics, we have
\begin{equation}
\begin{split}
&\mathbb{P}_{\mathcal{H}_0}\left({\sum_{b=1}^B\mathbf{1}\{\rho(X^{(b)},Y,Z)\geq \rho(X,Y,Z)\}}\leq \lfloor\alpha({1+B})\rfloor\right) \\& = \frac{\lfloor\alpha({1+B})\rfloor}{1+B}  \leq \alpha,   
\end{split}
\end{equation}
which completes the proof. 
\end{proof}
\noindent\textbf{Remark.} Note that all the above results assume that for all $b\in [B]$, $(X^{(b)},Y,Z)~{\buildrel d \over =}~(X,Y,Z)$, that the distribution of the generated samples perfectly models the conditional distribution. However, in practical applications, the generative model may introduce some error, even if we have provided an upper bound on this error, as analyzed in detail in Sec.~\ref{a1} and Sec.~\ref{a2}. Therefore the actual $p$-value estimate will have some deviation compared to the theoretical value caused by the estimation error. As a result, in the next Sec.~\ref{a4}, we will further examine the validity of the $p$-value within our CI testing framework and obtain the Type I error Bound.

\section{Proof of Type I error Bound}
\label{a4}
To simplify, we separate the samples used in the previous stage of generative modeling from the samples used in the CI test, and the number of samples is denoted as $N$ and $n$, respectively. We denote the estimated conditional distribution as $p_{T;N}(x|z)$. We further define $\boldsymbol{X}:=(x_1, x_2,...,x_n)^T, \boldsymbol{Y}:=(y_1, y_2,...,y_n)^T$ and $ \boldsymbol{Z}:=(z_1, z_2,...,z_n)^T$ as the vectors formed by $n$ samples. Additionally, for $b\in[B]$, we define $\boldsymbol{X}^{(b)}:=\bigl(x_1^{(b)}, x_2^{(b)},...,x_n^{(b)}\bigl)^T$ that are the generated vector corresponding to $\boldsymbol{Z}$. Then the estimation of statistic is given by $\hat{\rho}=\rho(\boldsymbol{X}, \boldsymbol{Y}, \boldsymbol{Z})$. Recall that we have defined the estimation of threshold in the main paper as $c_\alpha:=\inf\{c\in \mathbb{R}: \mathbb{P}(\hat{\rho}>c)\leq \alpha\}$. The following results give a bound for Type I error given $\boldsymbol{Y}, \boldsymbol{Z}$.  

\begin{lemma}[Type I error bound given $\boldsymbol{Y}, \boldsymbol{Z}$]
\label{theorem_type_I_conditional}
Assume $\mathcal{H}_0: X\perp \!\!\! \perp Y|Z$ is true. Under all the Assumptions in Sec.~\ref{a1} and~\ref{a2}, for any signifiance level $\alpha\in (0,1)$, given $\boldsymbol{Y}, \boldsymbol{Z}$, the bound for the Type I error is obtained as 
\begin{equation}
\begin{split}
\mathbb{P}_{\mathcal{H}_0}(\hat{\rho} > c_{\alpha}|\boldsymbol{Y}, \boldsymbol{Z}) &\leq \alpha + d_\text{TV}\{p_{T;N}(\cdot|\boldsymbol{Z}), p(\cdot|\boldsymbol{Z})\}\\
&= \alpha + o_p(1), \text{ as } N\to \infty.       
\end{split}
\end{equation}
\end{lemma}
\begin{proof}
By definition, the statistic $\hat{\rho}$ results in a $p$-value $< \alpha$ if and only if the observed variables are contained in the set $A_{\alpha}^B$, where each element $(\boldsymbol{x}, \boldsymbol{x}^{(1)}, ..., \boldsymbol{x}^{(B)})$ satisfies  
\begin{equation}
\frac{1}{B+1}\Bigl[1+\sum^B_{b=1} \mathbf{1}\{\rho(\boldsymbol{x}^{(b)}, \boldsymbol{Y}, \boldsymbol{Z})\geq \rho(\boldsymbol{x}, \boldsymbol{Y}, \boldsymbol{Z})\}\Bigl]< \alpha.     
\end{equation}
Let $\hat{\boldsymbol{X}}\sim p_{T;N}(\cdot|\boldsymbol{Z})$ be sampled from the estimated conditional distribution. Then it holds that, 
\begin{equation}
\begin{split}
\mathbb{P}_{\mathcal{H}_0}&(\hat{\rho} > c_{\alpha}|\boldsymbol{Y}, \boldsymbol{Z}) = \mathbb{P}_{\mathcal{H}_0}\left((\boldsymbol{X}, \boldsymbol{X}^{(1)},...,\boldsymbol{X}^{(B)})\in A_{\alpha}^B|\boldsymbol{Y}, \boldsymbol{Z}\right)\\
& = \mathbb{P}_{\mathcal{H}_0}\left((\hat{\boldsymbol{X}}, \boldsymbol{X}^{(1)},...,\boldsymbol{X}^{(B)})\in A_{\alpha}^B|\boldsymbol{Y}, \boldsymbol{Z}\right) \\ &\ \ \ \ + \mathbb{P}_{\mathcal{H}_0}\left((\boldsymbol{X}, \boldsymbol{X}^{(1)},...,\boldsymbol{X}^{(B)})\in A_{\alpha}^B|\boldsymbol{Y}, \boldsymbol{Z}\right) \\ &\ \ \ \ - \mathbb{P}_{\mathcal{H}_0}\left((\hat{\boldsymbol{X}}, \boldsymbol{X}^{(1)},...,\boldsymbol{X}^{(B)})\in A_{\alpha}^B|\boldsymbol{Y}, \boldsymbol{Z}\right)\\
\leq &\mathbb{P}_{\mathcal{H}_0}\left((\hat{\boldsymbol{X}}, \boldsymbol{X}^{(1)},...,\boldsymbol{X}^{(B)})\in A_{\alpha}^B|\boldsymbol{Y}, \boldsymbol{Z}\right) \\ &~~~~~~+ d_{\text{TV}}\Bigl\{({\boldsymbol{X}}, \boldsymbol{X}^{(1)},...,\boldsymbol{X}^{(B)}|\boldsymbol{Y}, \boldsymbol{Z}),(\hat{\boldsymbol{X}}, \boldsymbol{X}^{(1)},...,\boldsymbol{X}^{(B)}|\boldsymbol{Y}, \boldsymbol{Z})\Bigl\}.
\end{split}
\end{equation}    
By the definition of $\hat{\boldsymbol{X}}$, perform the same analysis as in Proposition~\ref{app_exchangeabliliy_random_variable} and Corollary~\ref{app_exchangeabliliy_statistics}, we can show that
\begin{equation}
\rho(\hat{\boldsymbol{X}}, \boldsymbol{Y}, \boldsymbol{Z}), \rho(\boldsymbol{X}^{(1)}, \boldsymbol{Y}, \boldsymbol{Z}), ...., \rho(\boldsymbol{X}^{(B)}, \boldsymbol{Y}, \boldsymbol{Z})    
\end{equation}
are exchangeable conditionally on $\boldsymbol{Y}, \boldsymbol{Z}$. Hence by combining the property of rank test similar to  Proposition~\ref{app_valid_pvalue}, we obtain that 
\begin{equation}
\mathbb{P}_{\mathcal{H}_0}\left((\hat{\boldsymbol{X}}, \boldsymbol{X}^{(1)},...,\boldsymbol{X}^{(B)})\in A_{\alpha}^B|\boldsymbol{Y}, \boldsymbol{Z}\right)\leq \alpha.    
\end{equation}
And by the definition of $\text{TV}$ distance of probability measures, we can further calculating that
\begin{equation}
\begin{split}
&d_{\text{TV}}\Bigl\{({\boldsymbol{X}}, \boldsymbol{X}^{(1)},...,\boldsymbol{X}^{(B)}|\boldsymbol{Y}, \boldsymbol{Z}),(\hat{\boldsymbol{X}}, \boldsymbol{X}^{(1)},...,\boldsymbol{X}^{(B)}|\boldsymbol{Y}, \boldsymbol{Z})\Bigl\} \\
&= \frac{1}{2}\int \bigl|p_{\hat{\boldsymbol{X}}, \boldsymbol{X}^{(1)},...,\boldsymbol{X}^{(B)}|\boldsymbol{Y}, \boldsymbol{Z}}(\boldsymbol{x}, \boldsymbol{x}^{(1)}, ..., \boldsymbol{x}^{(B)}) \\& \ \ \ \ \ \ \ \  - p_{{\boldsymbol{X}}, \boldsymbol{X}^{(1)},...,\boldsymbol{X}^{(B)}|\boldsymbol{Y}, \boldsymbol{Z}}(\boldsymbol{x}, \boldsymbol{x}^{(1)}, ..., \boldsymbol{x}^{(B)})\bigr|d\boldsymbol{x}d\boldsymbol{x}^{(1)}\cdots d\boldsymbol{x}^{(B)} \\
& = \frac{1}{2}\int \left|p_{\hat{\boldsymbol{X}}|\boldsymbol{Y}, \boldsymbol{Z}}(\boldsymbol{x}) - p_{\boldsymbol{X}|\boldsymbol{Y}, \boldsymbol{Z}}(\boldsymbol{x}) \right|d\boldsymbol{x}\\&=d_{\text{TV}}\Bigl\{(\hat{\boldsymbol{X}}|\boldsymbol{Y}, \boldsymbol{Z}),({\boldsymbol{X}}|\boldsymbol{Y}, \boldsymbol{Z})\Bigl\}, 
\end{split}
\end{equation}
where the calculation is based on the property that $(\boldsymbol{x}, \boldsymbol{x}^{(1)}, ..., \boldsymbol{x}^{(B)})$ is independent of each other. As a result, we obtain the bound of Type I error rate as
\begin{equation}
\begin{split}
\mathbb{P}_{\mathcal{H}_0}(\hat{\rho} > c_{\alpha}|\boldsymbol{Y}, \boldsymbol{Z}) - \alpha &\leq  d_{\text{TV}}\Bigl\{(\hat{\boldsymbol{X}}|\boldsymbol{Y}, \boldsymbol{Z}),({\boldsymbol{X}}|\boldsymbol{Y}, \boldsymbol{Z})\Bigl\} \\&=d_\text{TV}\{p_{T;N}(\cdot|\boldsymbol{Z}), p(\cdot|\boldsymbol{Z})\}\\&\leq \sum_{i=1}^n d_\text{TV}\{p_{T;N}(\cdot|z_i), p(\cdot|z_i)\}, 
\end{split}
\end{equation}
then by combining Theorem~\ref{appenddix Error bound of conditional distribution}, we complete the proof.
\end{proof}
Next, we show that the Type I error rate can be unconditionally controlled, as shown in Theorem~\ref{app_theorem_type_I}.
\begin{theorem}[Type I error bound]
\label{app_theorem_type_I}
Assume $\mathcal{H}_0: X\perp \!\!\! \perp Y|Z$ is true. Under all the Assumptions in Sec.~\ref{a1} and~\ref{a2}, for any signifiance level $\alpha\in (0,1)$, the bound for the Type I error is obtained as
\begin{equation}
\mathbb{P}(p\text{-value}\leq \alpha |\mathcal{H}_0) \leq \alpha + o_p(1), \text{ as } N\to \infty.    
\end{equation}
\end{theorem}
\begin{proof}
Applying Theorem~\ref{theorem_type_I_conditional} and Lebesgue dominated convergence theorem, by marginalizing over $\boldsymbol{Y}, \boldsymbol{Z}$ and note that the TV distance is upper bounded by $1$, thus we have
\begin{equation}
\begin{split}
\mathbb{P}(p\text{-value}\leq \alpha |\mathcal{H}_0) &\leq \alpha + \mathbb{E}_{\boldsymbol{Z}}\bigl[d_\text{TV}\{p_{T;N}(\cdot|\boldsymbol{Z}), p(\cdot|\boldsymbol{Z})\}\bigl] \\&\leq  \alpha + o_p(1), \text{ as } N\to \infty, 
\end{split}
\end{equation}
which completes the proof.
\end{proof}
Thus, these theories prove that our test is valid in the sense that asymptotically the Type I error can be well controlled, and more intuitively, when the training samples for the generative model are sufficiently large, the resulting sample distribution is sufficiently close to the true distribution, so that the upper bound of the Type I error rate is precisely controlled to any given $\alpha$.

\section{Implementation Details.}
\label{a5}

\subsection{Details of Compared Methods}
The compared methods in our experiments are described below.
\begin{itemize}[leftmargin=5mm]
\item \textbf{CCIT}~\citep{sen2017model}: Transforms the CI test into a classification problem, leveraging powerful classifiers such as gradient-boosted trees.
\item \textbf{RCIT}~\citep{strobl2019approximate}: Approximates the kernel-based CI test (KCIT) using random Fourier features for scalability. 
\item \textbf{FCIT}~\citep{chalupka1804fast}: Performs a fast conditional independence test by comparing the mean squared errors (MSE) from regressing $Y$ on $X, Z$, versus regressing $Y$ on $Z$ along. 
\item \textbf{GCM}~\citep{shah2020hardness}: Computes a normalized statistic for conditional independence testing based on the sample covariance between the regression residuals of $X$ and $Y$ given $Z$.
\item \textbf{KCIT}~\citep{zhang2012kernel}: A kernel-based CI test that constructs test statistics using kernel embeddings of the distributions.
\item \textbf{LPCIT}~\citep{scetbon2022asymptotic}: Measures conditional dependence by evaluating differences between analytic kernel embeddings of distributions at a finite set of locations.
\item \textbf{GCIT}~\citep{bellot2019conditional}: Employs generative adversarial networks (GANs) to model conditional distributions for CI testing. 
\item \textbf{DGCIT}~\citep{shi2021double}: Uses two GANs to model the conditional distributions $\mathbb{P}(X|Z)$ and $\mathbb{P}(Y|Z)$, and designs a random mapping-based statistic using neural networks.
\item \textbf{NNLSCIT}~\citep{li2023nearest}: Integrates a classifier-based conditional mutual information estimator. A $k$-nearest-neighbor local sampling strategy is used to approximate the null hypothesis samples.
\end{itemize}

\noindent Below are the GitHub URLs of the compared methods:
\begin{itemize}[leftmargin=5mm]
\item \textbf{CCIT}~\url{https://github.com/rajatsen91/CCIT}.
\item \textbf{RCIT}:~\url{https://github.com/ericstrobl/RCIT}.
\item \textbf{FCIT}:~\url{https://github.com/kjchalup/fcit}.
\item \textbf{GCM}: The R package is available.
\item \textbf{KCIT}:~\url{http://people.tuebingen.mpg.de/kzhang/KCI-test.zip}.
\item \textbf{LPCIT}:~\url{https://github.com/meyerscetbon/lp-ci-test}.
\item \textbf{GCIT}:~\url{https://github.com/alexisbellot/GCIT}.
\item \textbf{DGCIT}:~\url{https://github.com/tianlinxu312/dgcit}.
\item \textbf{NNLSCIT}:~\url{https://github.com/LeeShuai-kenwitch/NNLSCIT}.
\end{itemize}

\subsection{Details of SGMCIT}
We give the detailed implementation of SGMCIT as follows. 

\noindent\textbf{Model Architecture.}
The conditional score model is based on a multi-layer perceptron (MLP) with three fully connected layers, each followed by Swish activations. The input size is $d_x + d_z$, while the output size is $d_x$, with a hidden layer dimension of $64$. 

\noindent\textbf{Projection Vectors.} We set the distribution of the projection vectors to be Gaussian and set the projection number $m = 1$.

\noindent\textbf{Hyperparameters.} The model is trained using the Adam optimizer with a learning rate of $1 \times 10^{-4}$ over $100$ epochs. The batchsize is set to $50$ by default.  

\noindent \textbf{Data Preprocessing.}
To normalize the input data, we apply min-max scaling, transforming all features into the $[0, 1]$ range. Following this, a logit transformation is applied: $\log(x) - \log(1 - x)$.

\noindent \textbf{Sampling Procedure.}
The Langevin dynamics conditional sampling process is governed by the step size $h$ and the total number of steps. In this work, we set the step size $h$ to $0.1$ and the number of steps to $200$. For the samples used to model the null hypothesis distribution, we take a parallel implementation, i.e., we generate $B$ samples in parallel, which greatly reduces the time required to generate the samples utilizing parallel hardware.

\noindent \textbf{Code.} For more details, our code is available at \url{https://github.com/jinchenghou123/SGMCIT}.

\section{More Experiments Results}
\label{a6}
In this section, we provide more experimental results, including some visualization results, as well as results under two additional metrics, with the runtime results.

\subsection{Detailed Visualization Results}
In this section, we give the detailed visualization results of generative model based methods GCIT, DGCIT and SGMCIT (Ours). The results corresponding to benchmark datasets and high-dimensional confounder setting in the main paper.

\noindent \textbf{Visualization results on Benchmarks}. Figs.~\ref{fig:vissocit},~\ref{fig:visgcit}, and~\ref{fig:visdgcit} demonstrate the performance of SGMCIT, GCIT, and DGCIT in estimating the marginal distribution under different transformation functions: linear, square, cos, tanh, and exp functions. SGMCIT consistently performs well across all settings, with the marginal distributions of the generated samples closely matching those of the observed data. In contrast, both GCIT and DGCIT struggle with highly non-linear transformations, such as cosine and exponential functions.

\noindent \textbf{Visualization results of high-dimensional confounder setting}. Fig.~\ref{fig:vissocit-hd} compares the performance of GCIT, DGCIT, and SGMCIT in the high-dimensional confounder setting. In this case, both SGMCIT and DGCIT handle the high-dimensional setting effectively, outperforming GCIT. A further analysis of the approximation performance across different regions of the probability density shows that SGMCIT provides accurate approximations in various density regions. In contrast, while DGCIT yields similar overall distribution, it struggles with local accuracy, reflecting its inability to model the conditional distribution in certain areas. 

Overall, SGMCIT outperforms both GCIT and DGCIT in most scenarios, highlighting its ability to model complex distribution. Also, the visualization results provide interpretability for the performance of the CI testing in corresponding experimental results.

\subsection{Results under Additional Metrics}
In this section, we present experimental results using two additional metrics. For context, the main paper focuses on Type I and Type II errors, here we introduce a total of four evaluation metrics.

\noindent \textbf{Performance metrics.}
We assess performance using four metrics: (1) Type I error rate, which measures validity by ensuring the error rate remains controlled at any significance level $\alpha$; (2) Testing power, defined as $1 -$ Type II error rate, reflecting the ability to detect conditional dependencies; (3) Kolmogorov-Smirnov (KS) statistic, which under $\mathcal{H}_0$ compares the $p$-value distribution to a uniform [0,1], with smaller values indicating better uniformity; and (4) Area under the power curve (AUPC), which measures the empirical cumulative distribution of $p$-values under $\mathcal{H}_1$, with values closer to one indicating higher power. 

\noindent\textbf{Results and analysis.} 
The results for Cases 1 and 2 are shown in Fig.~\ref{fig:experiments_gaussian1}, while those for Cases 3 and 4 are shown in Fig.~\ref{fig:experiments_gaussian2}.

Across all metrics, SGMCIT excels at controlling Type I errors while maintaining high testing power across a variety of function combinations, establishing it as the most reliable method in these experiments. The KS statistic for SGMCIT demonstrates good uniformity of the $p$-value distribution across a large number of function combination settings, reflecting its effective modeling of the conditional distribution. The AUPC results align closely with the power results, further showcasing SGMCIT’s high power. In comparison, while most other methods perform well with the linear and tanh functions, they struggle with some other settings. For instance, DGCIT often fails to control Type I errors effectively, CCIT shows weak performance in terms of testing power, and GCIT exhibits poor $p$-value uniformity.

\subsection{Experimental Results of Running Time}
This section presents the results of running time for each method. All experiments are performed on the same device. The runtime for a single test is reported in a high-dimensional confounder setting with standard Gaussian noise.

\noindent\textbf{Results.} Fig.~\ref{fig:runtime} shows the performance of all methods. When varying the sample size, we fixed $d_z = 100$, while for varying dimensionality, we set the sample size to 1000.

\noindent\textbf{Analysis.} SGMCIT, RCIT, GCM, and GCIT exhibit consistently low runtime, demonstrating strong scalability with respect to sample size. KCIT and LPCIT stand out with significantly longer runtime as the sample size increases. For example, KCIT exceeds 300 seconds for 10,000 samples, while LPCIT approaches 500 seconds. 

SGMCIT and GCIT maintain low and stable runtime across all dimensions, demonstrating their efficiency in high-dimensional settings. This can be attributed to the full utilization of parallel hardware by the generative model. DGCIT, while also utilizing generative models, has a longer overall runtime due to the multiple models that need to be trained as well as ineffective statistic design. Notably, LPCIT shows exponential growth in runtime, becoming the slowest method as the dimensionality exceeds 60. For instance, at the 100-dimensional setting, LPCIT's runtime exceeds 200 seconds. 

These results demonstrate that SGMCIT is computationally efficient, handling both large sample sizes and high-dimensional conditioning sets effectively. Among generative model-based CI methods, SGMCIT performs the best in terms of runtime efficiency.

\subsection{Additional Baseline Results on Real Data}

We evaluated all baseline methods on a real-world dataset, using the same experimental setup as in previous sections with a test sample size of 1000. Results are summarized in Table~\ref{tab:real_pval}.
\begin{table}[h]
\caption{
\normalsize{The CI results of 10 methods on real datasets.
}}
\begin{center}
\scalebox{0.9}{
\begin{tabular}{l|l}
\toprule
Method     & $p$-value\\ 
\midrule
CCIT &  0.68 \\ 
RCIT & 0.00 \\ 
FCIT & 0.03 \\
GCM  &  0.00 \\
KCIT &  0.00 \\
LPCIT & 0.00 \\
GCIT &  0.00 \\
DGCIT & 0.00 \\
NNLSCIT & 0.27  \\
SGMCIT & 0.00 \\
\bottomrule
\end{tabular}
}
\end{center}
\label{tab:real_pval}
\end{table}

\noindent\textbf{Analysis.} Although the ground truth is unknown, most methods—including our proposed SGMCIT—reject the null hypothesis, indicating that $X$ and $Y$ are not conditionally independent given $Z$. This aligns with our model's conclusion. In contrast, CCIT and NNLSCIT produce higher $p$-values. However, their poor power in synthetic experiments, where the ground truth is known, suggests that these results may be less reliable.

\begin{figure*}
    \centering
    \includegraphics[scale=0.74]{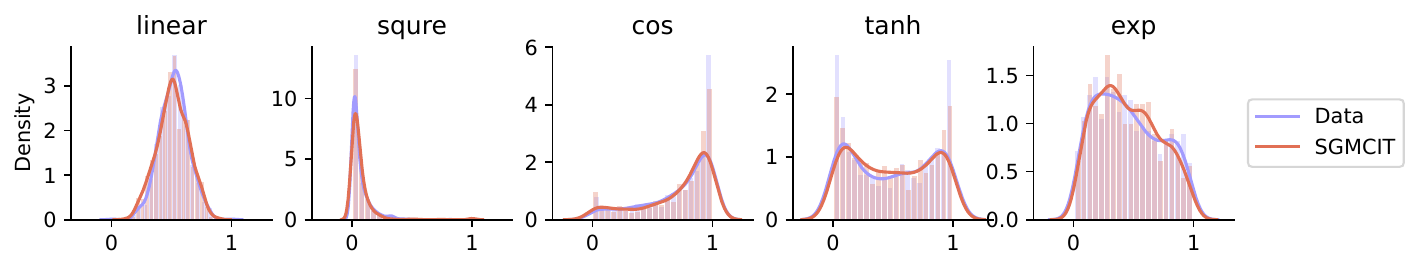}
    \caption{The visualization results of SGMCIT for the marginal distribution of $X$ under Case 4 of benchmark datasets.}
    \label{fig:vissocit}
\end{figure*}

\begin{figure*}
    \centering
    \includegraphics[scale=0.73]{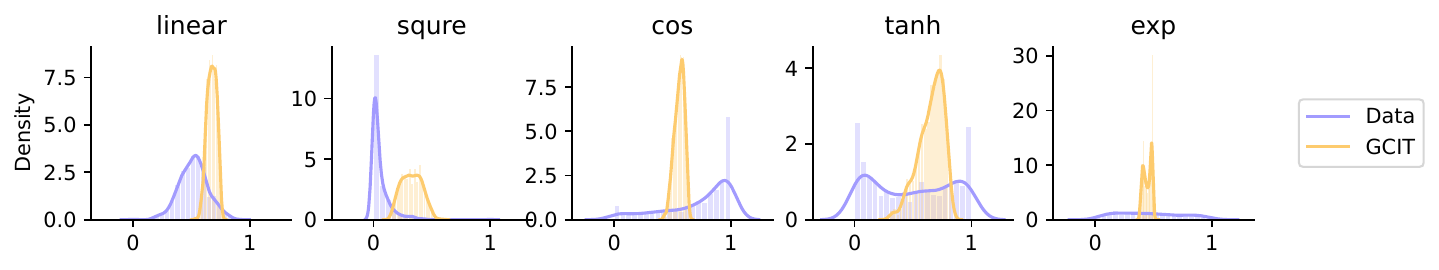}
    \caption{The visualization results of GCIT for the marginal distribution of $X$ under Case 4 of benchmark datasets.}
    \label{fig:visgcit}
\end{figure*}

\begin{figure*}
    \centering
    \includegraphics[scale=0.73]{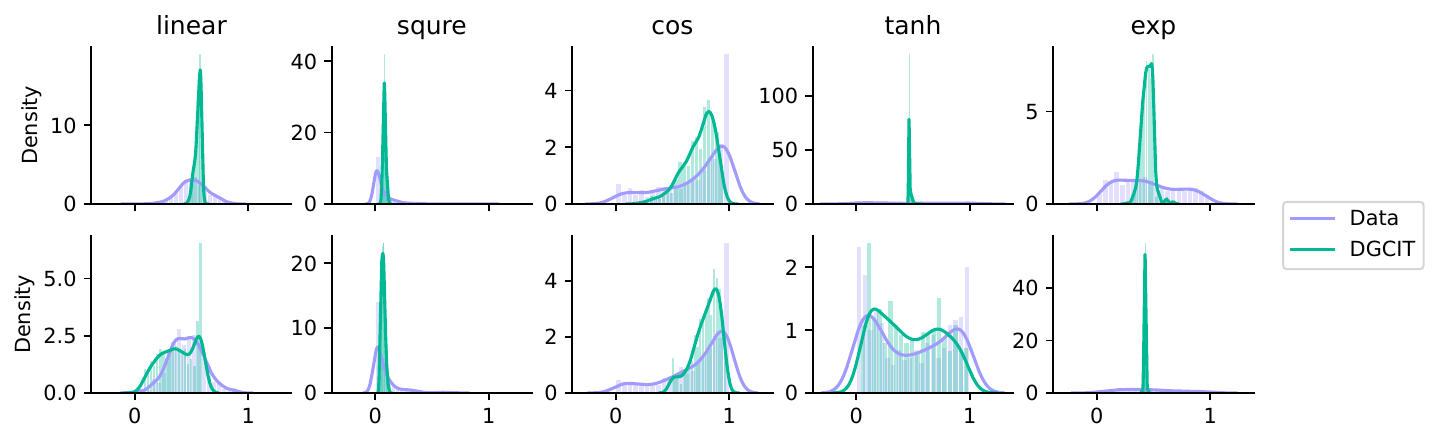}
    \caption{The visualization results of DGCIT for the marginal distribution of $X$ and $Y$ under Case 4. Top row: the results of $X$. Below row: the results of $Y$.} 
    \label{fig:visdgcit}
\end{figure*}

\begin{figure*}
    \centering
    \includegraphics[scale=0.92]{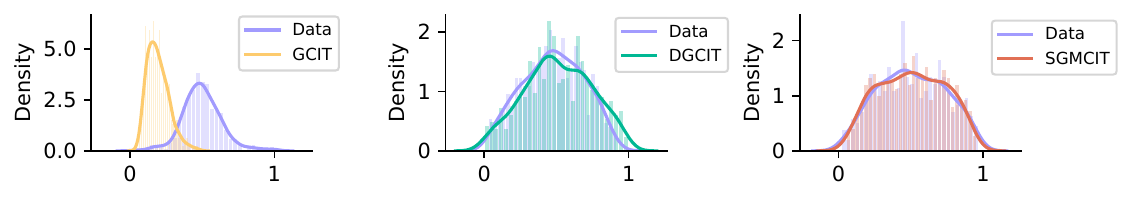}
    \caption{The visualization results of GCIT, DGCIT and SGMCIT for the marginal distribution of $X$ under high-dimensional confounder setting.}
    \label{fig:vissocit-hd}
\end{figure*}

\begin{figure*}[h]
\centering
\includegraphics[scale=0.51]{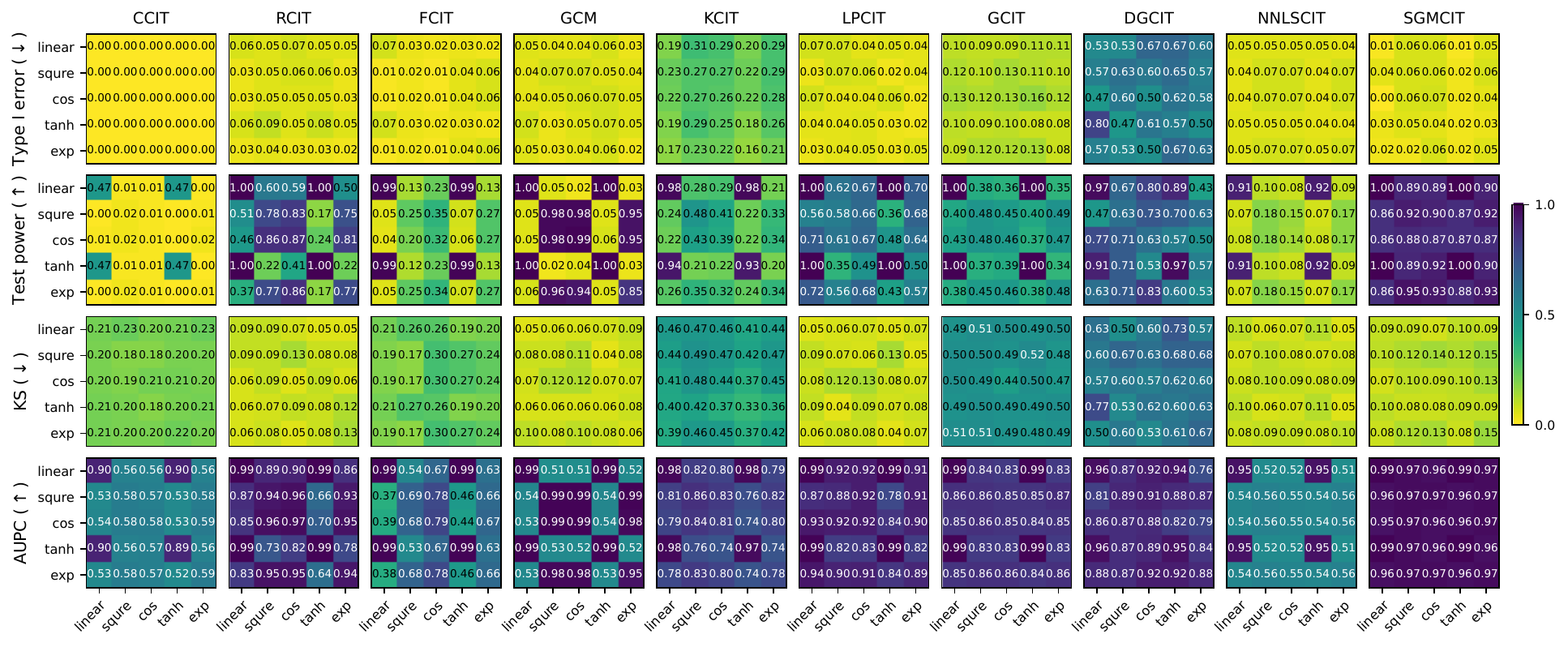}
\caption{
Additional results of conditional independence tests for Cases 1 and 2 on benchmark datasets.}
\label{fig:experiments_gaussian1}
\end{figure*}

\begin{figure*}[h]
\centering
\includegraphics[scale=0.51]{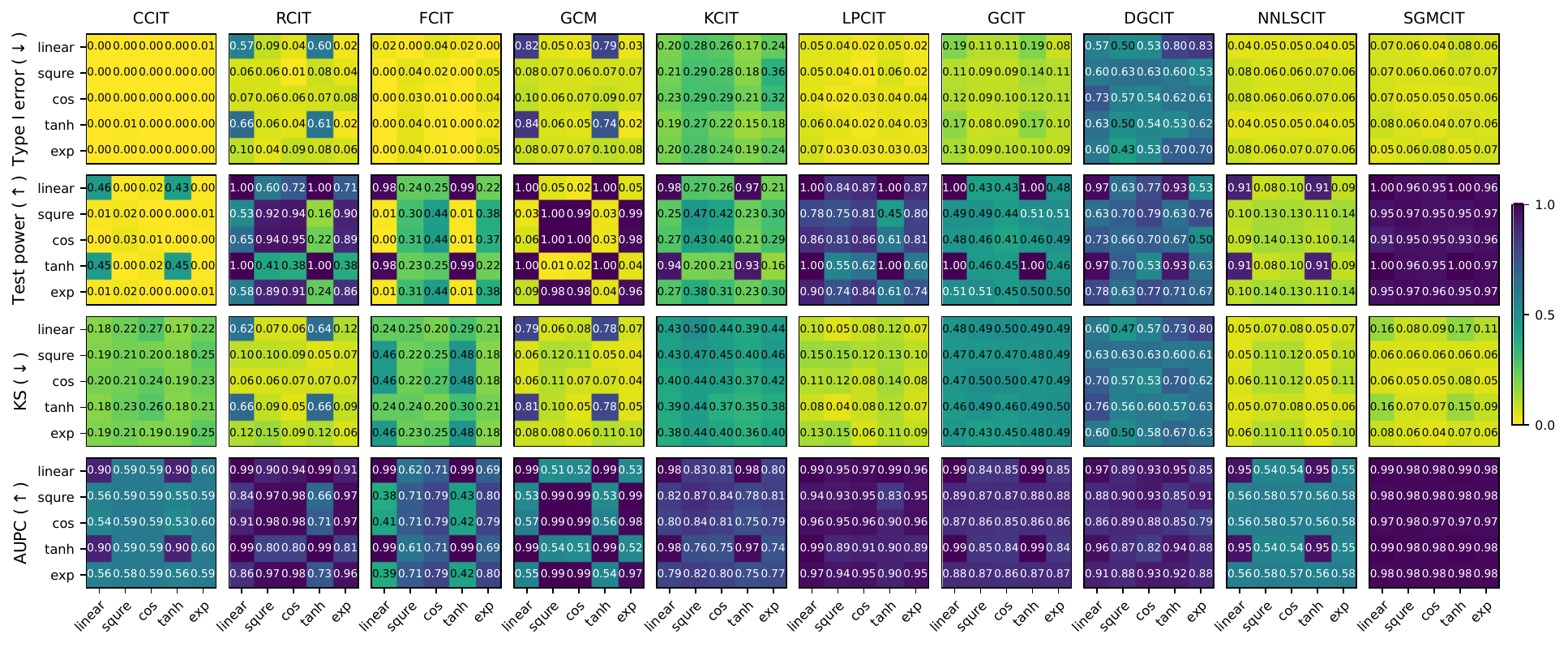}
\caption{
Additional results of conditional independence tests for Cases 3 and 4 on benchmark datasets. }
\label{fig:experiments_gaussian2}
\end{figure*}

\begin{figure*}[h]
\centering
\includegraphics[scale=0.28]{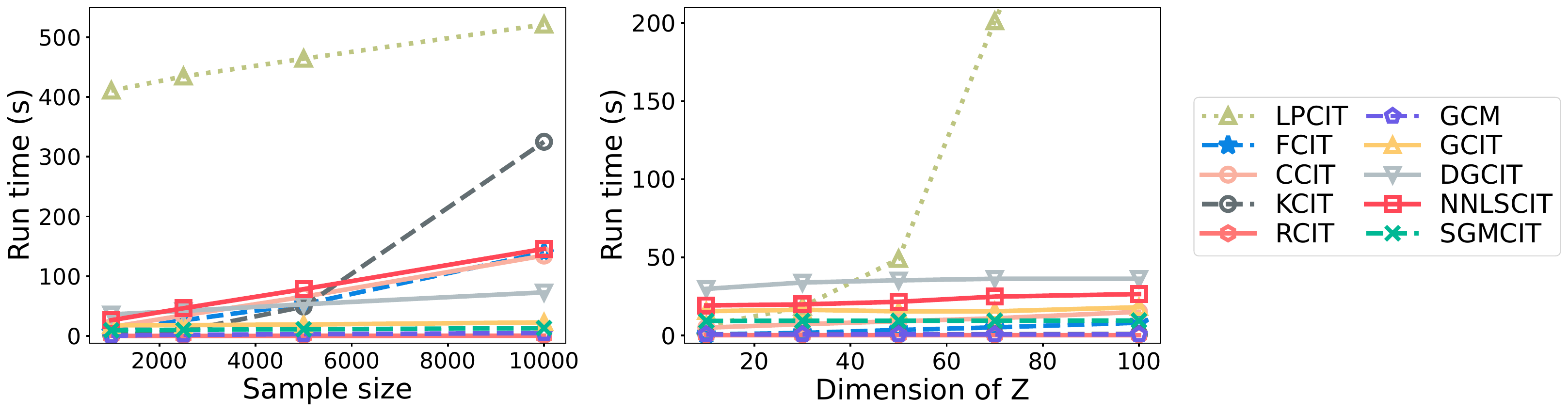}
\caption{Results of running time. \textbf{Left:} The results w.r.t. the sample size. \textbf{Right:} The results w.r.t. the dimension of $Z$.}
\label{fig:runtime}
\end{figure*}

\subsection{Comparison with CDCIT~\citep{yang2025conditional}}
As noted in Sec.~\ref{sec:post-acceptance remarks}, we provide a thorough comparison between our proposed method, SGMCIT, and the contemporaneous work CDCIT~\citep{yang2025conditional}. Unless otherwise specified, all experiments follow the same setup described in the main paper. This comparison includes evaluations on benchmark datasets, high-dimensional settings, and real-world data. Furthermore, we analyze each method's ability to estimate marginal distributions and compare their computational efficiency. CDCIT is implemented using default settings. 

\subsubsection{Results on benchmark datasets}
We begin by comparing SGMCIT and CDCIT on four standard benchmark cases. In Fig.~\ref{fig:cdcit-1-4}(a), we report results on Cases 1 and 2, while Fig.~\ref{fig:cdcit-1-4}(b) presents results for Cases 3 and 4. These cases encompass a range of functional forms and dependency strengths, designed to systematically evaluate both Type I error control and statistical power.
Also, we present CDCIT's visualization results under different transformation functions: linear, square, cos, tanh, and exp functions in Fig.~\ref{fig:cdcit_vis_syn}. 

\noindent\textbf{Analysis.} Our results show that CDCIT struggles in multiple aspects. It fails to control the Type I error in several settings, which undermines its reliability as a statistical test. More critically, its test power remains consistently low—even when the conditional dependency between variables is strong and should be easily detectable. This indicates a limited sensitivity to genuine signals in the data. Additionally, CDCIT often fails to accurately estimate the marginal distribution, particularly under nonlinear transformations such as $\cos$, $\tanh$, or $\exp$. Even in relatively simple cases—such as a linear transformation—its performance is at best moderate. These findings suggest that the conditional diffusion model employed by CDCIT has difficulty modeling complex or high-frequency distributions.

In contrast, SGMCIT consistently performs well across all benchmark cases. It not only achieves strong Type I error control, but also maintains high test power across a variety of functional forms. The generative component of SGMCIT produces accurate marginal estimates even in challenging scenarios, highlighting its effectiveness in modeling complex dependencies and distributions.

\subsubsection{Results on High-Dimensional setting}
We further examine performance under high-dimensional confounding setting ($d_z = 100$). The results are provided in Fig.~\ref{fig:cdcit-1-4} (c). 

\noindent\textbf{Analysis.} It can be observed that CDCIT performs poorly in high-dimensional settings, CDCIT exhibits consistently low test power, failing to detect dependencies on conditioned on $Z$ even when they are pronounced. For example, in the case where $Z$ is high-dimensional and the strength of dependence is strong (i.e.,  $b = 0.6$), CDCIT still yields unsatisfactory results. This suggests that CDCIT may suffer from an inherent inability to capture intricate conditional relationships in high-dimensional environments.

In contrast, our proposed method SGMCIT maintains excellent performance even under these challenging conditions. It achieves both strong Type I error control and high test power, demonstrating robust behavior regardless of the dimensionality of the input variables. These results highlight the advantage of our approach in practical applications where high-dimensional data is common and effective CI testing is critical.



\subsubsection{Visualization Results on Real Data}
To further assess CDCIT’s generative capacity in practical scenarios, we visualize its estimated marginal distributions on a real-world dataset. The results are presented in Fig.~\ref{fig:cdcit_vis_real}.

\noindent\textbf{Analysis.} To compensate for this, we provided CDCIT with a significantly larger sample size ($n=50,000$) during training on real-world datasets. However, even with this increased data availability, the estimated marginal distributions remained inaccurate, demonstrating that simply increasing the sample size is insufficient to overcome the method’s inherent limitations.

In contrast, SGMCIT achieves highly accurate marginal distribution estimation. This highlights not only its modeling capacity but also its efficiency in data usage.

\subsubsection{Running Time Evaluation}
We evaluate the computational efficiency of CDCIT compared to other baseline methods. The timing results are shown in Fig.~\ref{fig:cdcit_runtime}.

\noindent\textbf{Analysis.} CDCIT is computationally expensive. Even in favorable conditions with low sample size ($n = 1000$) and moderate dimensionality ($d_z = 10$), a single run of CDCIT takes nearly 40 seconds. This is significantly slower than most other methods evaluated. Such runtime requirements may render CDCIT impractical for large-scale or time-sensitive applications.

Our method, SGMCIT, on the other hand, is far more efficient. It achieves faster execution while maintaining high statistical performance, making it well-suited for real-world tasks where both accuracy and speed are essential.





\begin{figure*}[h]
    \centering
    \begin{subfigure}[b]{0.3\textwidth}
        \centering
        \includegraphics[scale=0.65]{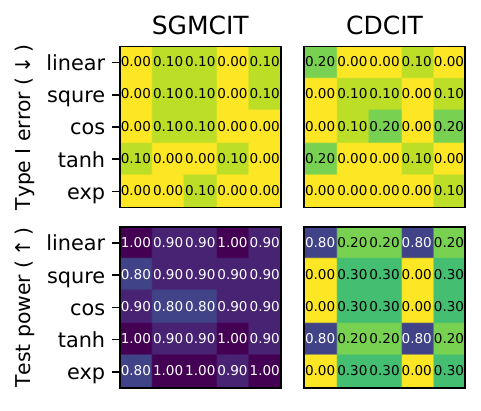}
        \caption{}
        \label{fig:cdcit1}
    \end{subfigure}
    \hfill
    \begin{subfigure}[b]{0.3\textwidth}
        \centering
        \includegraphics[scale=0.65]{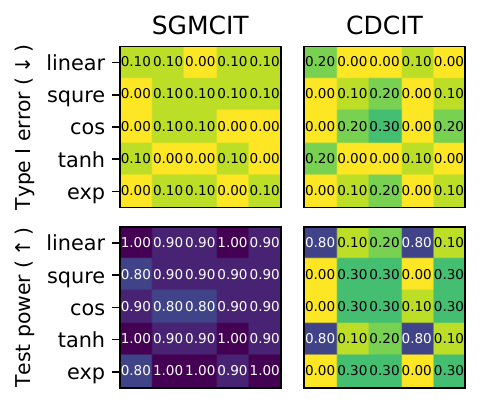}
        \caption{}
        \label{fig:cdcit2}
    \end{subfigure}
    \hfill
    \begin{subfigure}[b]{0.36\textwidth}
        \centering
        \includegraphics[scale=0.25]{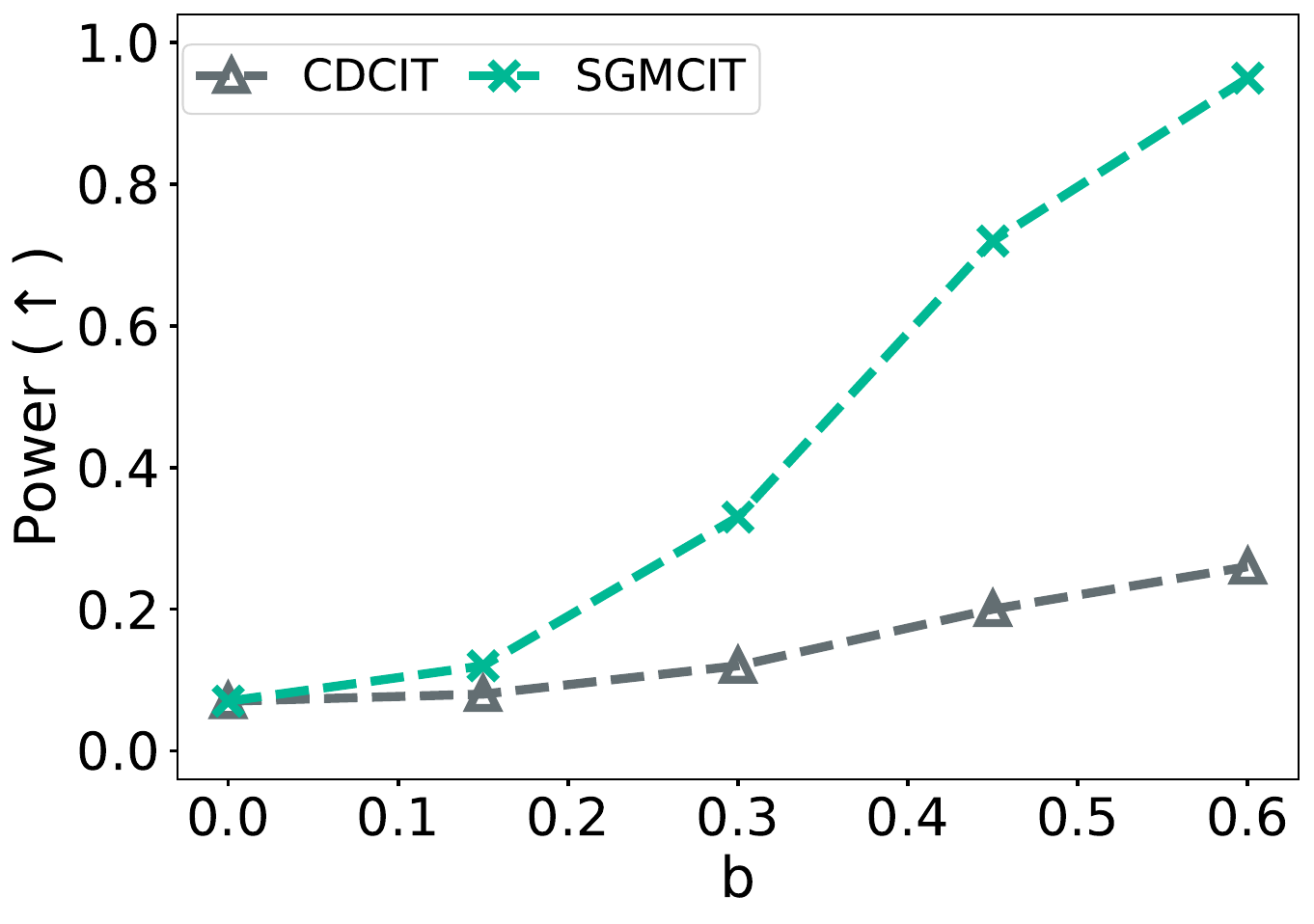}
        \caption{}
        \label{fig:cdcit_high}
    \end{subfigure}
    \caption{Additional results of CI tests. (a) Additional results of CI tests for Cases 1 and 2 on benchmark datasets. (b) Additional results of CI tests for Cases 3 and 4 on benchmark datasets. (c) Results in the high-dimensional confounder setting.}
    \label{fig:cdcit-1-4}
\end{figure*}




\begin{figure*}[h]
\centering
\includegraphics[scale=0.745]{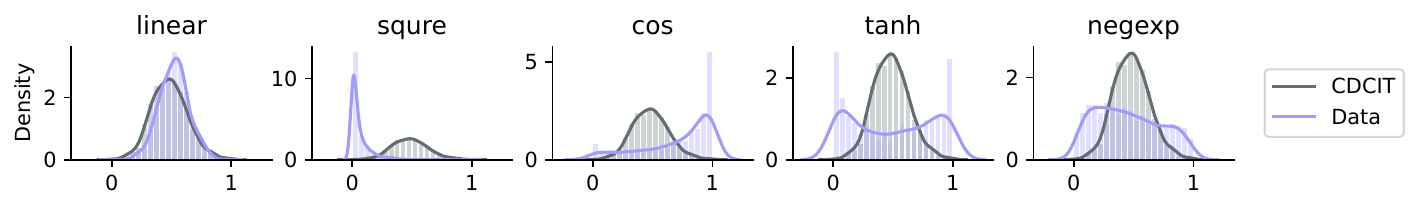}
\caption{The visualization results of SGMCIT and CDCIT for the  distribution of $X$ under Case 4 of benchmark datasets.}
\label{fig:cdcit_vis_syn}
\end{figure*}

\begin{figure*}[h]
\centering
\includegraphics[scale=1.03]{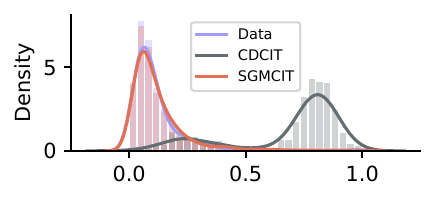}
\caption{Visualization results of SGMCIT and CDCIT on real data with 50000 sample size.}
\label{fig:cdcit_vis_real}
\end{figure*}

\begin{figure*}[h]
\centering
\includegraphics[scale=0.32]{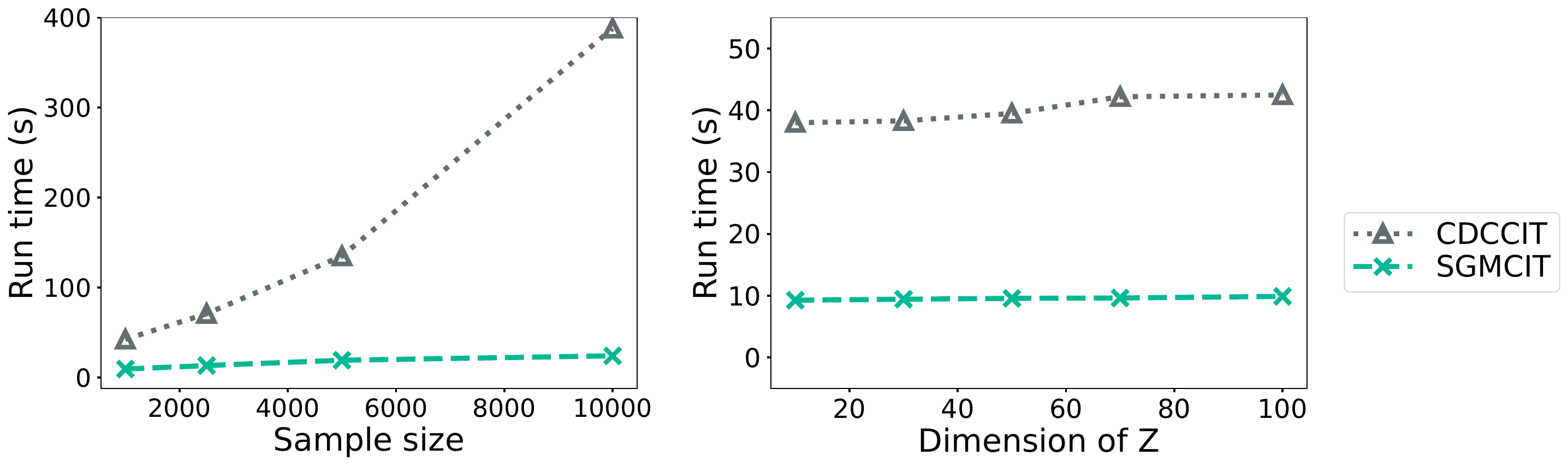}
\caption{Results of running time of CDCIT and SGMCIT (Ours). Left: The results w.r.t. the sample size. Right: The results w.r.t. the dimension of $Z$.}
\label{fig:cdcit_runtime}
\end{figure*}

\end{document}